\documentclass[journal]{IEEEtran}
\usepackage{makecell}
\usepackage[OT1]{fontenc}
\usepackage{adjustbox}
\usepackage{algorithm}
\usepackage{algorithmic}
\usepackage{amsfonts}
\usepackage{hhline}
\usepackage{amsmath}
\usepackage{amssymb}
\usepackage{amsthm}
\usepackage{arydshln}
\usepackage{balance}
\usepackage{bbm}
\usepackage{bm}
\usepackage{booktabs}
\usepackage{array}
\usepackage[font=small]{caption}
\usepackage{color}
\usepackage{cite}

\usepackage{enumitem}
\usepackage{float}
\usepackage{grffile}
\usepackage{graphicx}
\usepackage{graphics}
\usepackage{microtype}
\usepackage{multirow}
\usepackage{url}
\usepackage[table]{xcolor}
\usepackage{xparse}
\usepackage{colortbl}
\usepackage{threeparttable}
\usepackage{dsfont}
\usepackage{hyperref}
\usepackage{float}
\usepackage{graphicx}
\usepackage{multirow}
\usepackage{xcolor}
\usepackage{arydshln}
\newcommand{\fail}[1]{\textcolor{red}{#1}}

\newtheorem{theorem}{Theorem}

\theoremstyle{definition}

\usepackage{shortex}

\graphicspath{{images/}}

\allowdisplaybreaks

\begin{document}



\title{Efficient Covariance-Aware Vicinal Latent Diffusion for Continuous Conditional Image Generation}

%
\author{Xin Ding,
Shuo Huang,
Zhenglong Ding,
Sen Zhang,
Nenglun Chen,
Peibei Cao,
Yongwei Wang,
and Fei Wu
		\thanks{This work was supported in part by the National Natural Science Foundation of China (Grant No.~62306147, 62676198, 62303231, 62306144, 62676197, U24A20326, 62502442, 62501293), the Basic Research Program of Jiangsu (Grant No.~BK20241901, BK20250724), Jiangsu Supplementary Funds of National Major Talent Project (2022-2024) (Grant No.~R2024PT03), and the Startup Foundation for Introducing Talent of NUIST (Grant No.~2024r056). 
        }
		\thanks{Xin Ding, Shuo Huang, Zhenglong Ding (Corresponding Author), Sen Zhang, Nenglun Chen, and Peibei Cao are affiliated with the School of Artificial Intelligence/School of Future Technology, Nanjing University of Information Science \& Technology (NUIST), Nanjing, China (e-mail: \{dingxin, 202512492768, zlding, senzhang, chennenglun, 003970\}@nuist.edu.cn).} 
		\thanks{Yongwei Wang is affiliated with College of Media, Zhejiang University, Hangzhou, China (e-mail: yongwei.wang@zju.edu.cn).}
		\thanks{Fei Wu is affiliated with College of Artificial Intelligence, Zhejiang University, Hangzhou, China (e-mail: wufei@zju.edu.cn).}
	  }
	

\maketitle

\begin{abstract}
Continuous Conditional Generative Modeling (CCGM) aims to synthesize high-dimensional data conditioned on continuously valued scalar labels. However, existing continuous diffusion models face severe data sparsity under imbalanced and long-tailed label distributions, which can lead to visual artifacts and condition inconsistency. Moreover, inference-time classifier-free guidance and dense pixel-space integration introduce substantial sampling overhead, limiting practical deployment. To address these challenges, we propose CAViLD, an efficient latent diffusion framework for continuous conditional generation. CAViLD first constructs a purified condition-aware latent space through density-adaptive vicinal learning, auxiliary semantic regularization, and NIQE-guided negative data augmentation. It then internalizes the conditional guidance direction through an EMA teacher, enabling single-branch conditional inference, and introduces covariance-aware progressive distillation for label-dependent probability-flow ODEs, supporting sampling down to four NFEs. Experiments on RC-49, UTKFace, and Steering Angle demonstrate a favorable trade-off among generation fidelity, continuous condition consistency, and sampling efficiency.
\end{abstract}


\section{Introduction}
\label{sec:introduction}
Continuous conditional generative modeling (CCGM) aims to learn
high-dimensional distributions conditioned on continuous scalar
labels, such as viewpoints, ages, steering angles, and physical
parameters \cite{ding2021ccgan,ding2023ccgan,ding2026ccdm}. It supports
engineering inverse design
\cite{heyrani2021pcdgan,zhao2024ccdpm,fang2023diverse},
remote-sensing and hyperspectral generation
\cite{giry2022sar,shi2025regression,zhu2023data}, and point-cloud or
continuously controllable image synthesis
\cite{triess2022point,ding2024turning}. More broadly, recent studies have emphasized the importance of information representation, transformation, and organization in the design of intelligent systems~\cite{xu2023foundations,he2024signal, zheng2025machine}.  Unlike discrete class conditioning,
exact continuous labels are sparsely observed in finite datasets, making local
conditional distributions difficult to estimate, particularly in long-tailed
regions \cite{ding2021ccgan,ding2025imbalance,ding2026ccdm}.

Recent methods, including CCDM \cite{ding2026ccdm} and its matrix-form
extension iCCDM \cite{ding2026iccdm}, improve continuous conditional
generation through label-dependent diffusion processes and covariance
schedules. Although iCCDM reformulates DDPM-style diffusion
\cite{ho2020denoising,nichol2021improved} within the EDM paradigm
\cite{karras2022elucidating}, both methods perform iterative reverse diffusion
in pixel space, resulting in higher computational cost than latent diffusion
models \cite{rombach2022high,zhao2024ccdpm}. This limitation grows at higher
resolutions and restricts deployment.


\begin{figure}[t]
    \centering
    \includegraphics[width=\columnwidth]{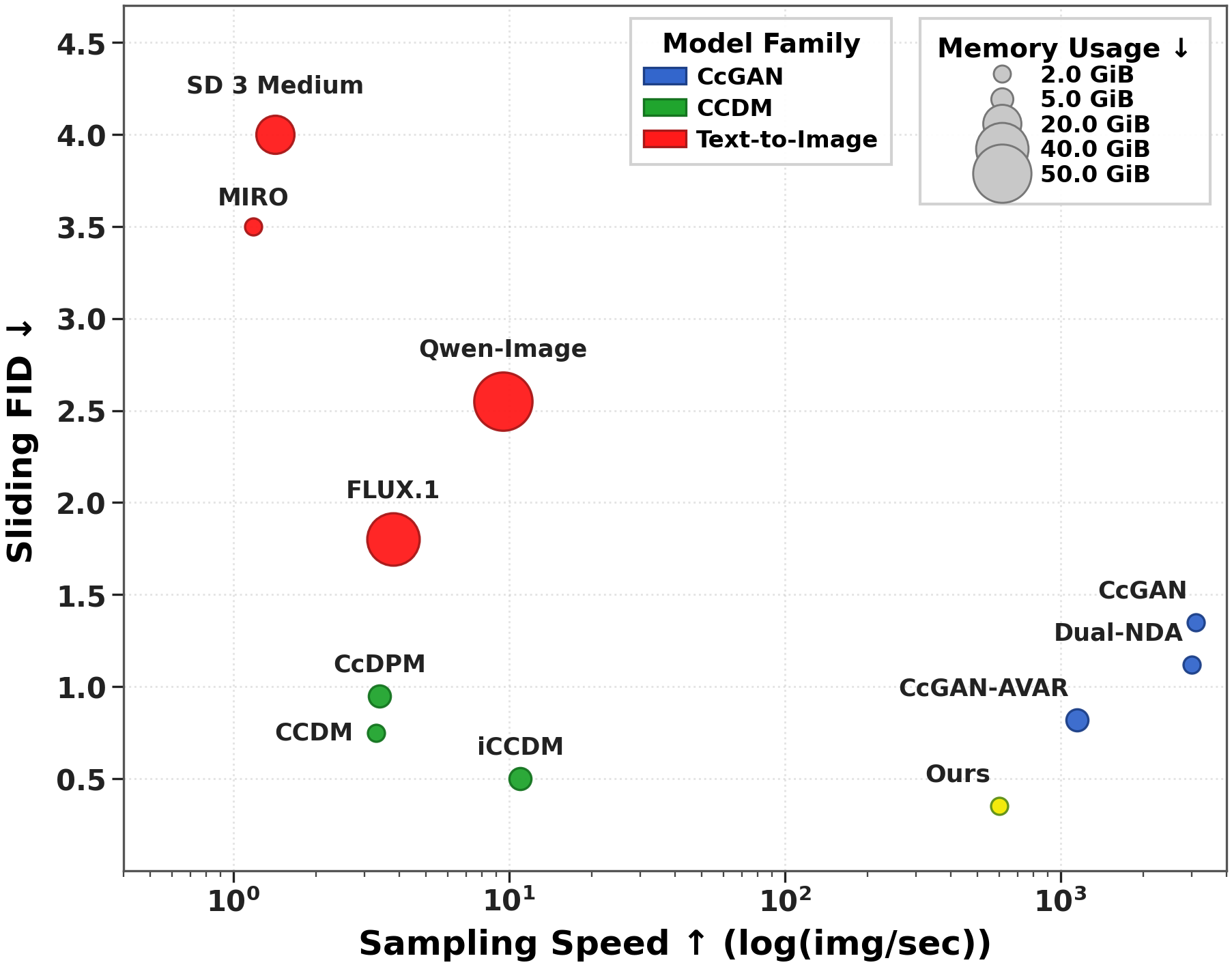}
    \caption{Quality--efficiency comparison on Steering Angle at
    $64\times64$. The axes denote sampling throughput and SFID, while marker size indicates aggregate peak GPU memory. CAViLD achieves a favorable trade-off
    among the evaluated diffusion-based CCGM methods.}
    \label{fig:benchmark}
\end{figure}

Moving continuous conditional diffusion into latent space raises three
challenges. First, standard conditional autoencoders do not explicitly address
sparse-label neighborhoods and may learn unreliable latent representations
\cite{sohn2015learning,ding2025imbalance}. Second, Classifier-Free Guidance
(CFG) requires conditional and unconditional denoiser evaluations at each
solver step, while large guidance scales may introduce artifacts or reduce
diversity
\cite{Ho2022ClassifierFreeDG,chung2025cfgpp,sadat2025eliminating}.
Third, few-step distillation methods are mainly developed for isotropic
diffusion \cite{salimans2022progressive,yin2024improved} and do not explicitly
account for the label-dependent anisotropic covariance of matrix-form
continuous diffusion \cite{ding2026iccdm}. Efficient latent CCGM therefore
requires jointly addressing representation quality, conditional guidance,
and trajectory compression.

To this end, we propose \textbf{CAViLD}, a Covariance-Aware Vicinal Latent
Diffusion framework with internalized guidance for continuous conditional
generation, building upon the matrix-form
formulation of iCCDM \cite{ding2026iccdm}. First, we learn a compact
conditional latent space using a Conditional Variational Autoencoder (CVAE)
\cite{sohn2015learning}. Density-aware vicinal learning
\cite{ding2025imbalance} improves sample coverage in sparse label regions,
while NIQE-guided Negative Data Augmentation (NDA)
\cite{ding2024turning} provides explicit supervision against low-quality
reconstruction modes.

Second, we incorporate Model-Guidance (MG) \cite{tang2025diffusion} into
matrix-form latent diffusion. During training, an exponential moving average
teacher transfers the conditional--unconditional guidance direction into the
denoising target. The trained model requires one conditional denoiser
evaluation per solver step, eliminating the additional unconditional forward
pass used by inference-time CFG.

Third, we introduce covariance-aware progressive distillation for the
probability-flow ODE \cite{song2021scorebased,chen2023probability}. A
condition-dependent precision-weighted objective accounts for the anisotropic
covariance, and step-halving \cite{salimans2022progressive} compresses the
trajectory to four NFEs. Experiments on RC-49, UTKFace, and Steering Angle
show a favorable trade-off among generation quality, label consistency, and
sampling efficiency.
Compared with prior CcGAN-, CCDM-, and iCCDM-based approaches
\cite{ding2021ccgan,ding2023ccgan,ding2024turning,ding2025imbalance,
ding2026ccdm,ding2026iccdm}, the framework jointly studies continuous
conditional latent representation, CFG-free single-branch inference, and
covariance-aware few-step sampling.

Our main contributions are summarized as follows:
\begin{itemize}
    \item \textbf{Continuous conditional latent diffusion.}
    We formulate matrix-form continuous conditional diffusion in a compact
    CVAE latent space and adapt density-aware vicinal learning and
    NIQE-guided NDA to improve representation quality under imbalanced labels.

    \item \textbf{Model-guided denoising.}
    We transfer the conditional guidance direction into the training target,
    eliminating the additional unconditional denoiser evaluation required by
    inference-time CFG.

    \item \textbf{Covariance-aware few-step distillation.}
    We derive a precision-weighted progressive distillation objective for
    label-dependent probability-flow ODEs, supporting sampling down to four NFEs.
\end{itemize}

\section{Related Work}
\label{sec:related_work}

\subsection{Continuous Conditional Image Generation}
\label{subsec:rel_ccgm}

Conditional image synthesis is well established for discrete categories
\cite{mirza2014conditional,odena2017conditional} or text prompts
\cite{rombach2022high,esser2024scaling}, whereas extending it to Continuous
Conditional Generative Modeling (CCGM) with continuously valued scalar
attributes $y\in\mathbb{R}$, such as design parameters and aspect angles
\cite{heyrani2021pcdgan,giry2022sar}, introduces a fundamental data-sparsity
problem. In a finite regression dataset, few or no observations may share an
exact target label, making direct estimation of the local conditional
distribution unreliable, especially in long-tailed regions
\cite{ding2023ccgan,ding2025imbalance,ding2026ccdm}.

CcGANs address this issue through vicinal empirical losses that borrow
information from neighboring labels
\cite{ding2021ccgan,ding2023ccgan}; Dual-NDA \cite{ding2024turning} and
CcGAN-AVAR \cite{ding2025imbalance} further improve negative supervision and
density-adaptive neighborhood construction. Diffusion alternatives, including
CcDPM \cite{zhao2024ccdpm}, CCDM \cite{ding2026ccdm}, and iCCDM
\cite{ding2026iccdm}, replace adversarial optimization with iterative
denoising. In particular, CCDM introduces label-dependent perturbation, while
iCCDM reformulates it using a matrix-form covariance
$\mathbf{\Sigma}(t,y)$ and anisotropic diffusion dynamics. Nevertheless,
long-tailed labels can still cause condition misalignment, and pixel-space
simulation remains computationally expensive. Latent modeling reduces the per-step computational cost by performing diffusion in a compact representation space \cite{rombach2022high}. However, in continuous conditional generation, latent compression must preserve both reconstruction fidelity and fine-grained condition semantics, particularly in sparsely sampled or imbalanced label regions where the learned representation may otherwise become unreliable \cite{sohn2015learning,ding2025imbalance}.

\subsection{Classifier-Free Guidance and Distillation}
\label{subsec:rel_cfg}

Conditional diffusion commonly employs classifier guidance
\cite{dhariwal2021diffusion} or Classifier-Free Guidance (CFG)
\cite{Ho2022ClassifierFreeDG}, with CFG extrapolating conditional predictions away
from an unconditional baseline. Evaluating both the conditional denoiser
$\mathcal{D}_\theta(z_t,t,y)$ and its unconditional counterpart
$\mathcal{D}_\theta(z_t,t,\emptyset)$ at each integration point approximately
doubles the network evaluations. Moreover, large CFG scales may move sampling trajectories away from the
learned manifold, producing mode collapse, oversaturation, artifacts, or
diversity loss \cite{chung2025cfgpp,sadat2025eliminating}. 

Model-Guidance (MG) \cite{tang2025diffusion} avoids the inference-time
unconditional branch by internalizing the conditional--unconditional direction
into the training objective. Sampling acceleration has also been studied
through implicit sampling \cite{song2021denoising}, high-order guided solvers
\cite{lu2022dpm_plus}, consistency modeling \cite{song23consistency}, and
progressive distillation \cite{salimans2022progressive}. The latter trains a
student to reproduce two consecutive teacher transitions with one step.
However, these methods are generally developed for scalar or isotropic noise
schedules and do not account for the continuous, condition-dependent
anisotropic geometry $\mathbf{\Sigma}(t,y)$. This motivates a
covariance-aware model-guided distillation formulation.



\section{Method}
\label{sec:method}
To improve the generation quality and sampling efficiency of
iCCDM~\cite{ding2026iccdm}, we extend its matrix-form continuous conditional
diffusion framework from pixel space to a compact latent space. We retain
the label-dependent covariance formulation, matrix-form probability-flow
dynamics, covariance embedding, and denoiser preconditioning of iCCDM, while
adapting its density-aware hard-vicinity principle to latent-space
construction. Relative to this backbone, CAViLD introduces three extensions:
a purified condition-aware CVAE latent space, internalized model guidance for
single-branch conditional inference, and covariance-aware progressive
distillation for few-step sampling. Specifically, the CVAE is first trained
and frozen; an EMA teacher then transfers the conditional--null guidance
direction into the latent denoiser; finally, the resulting single-branch
model is progressively distilled using covariance-aware endpoint matching.
Figure~\ref{fig:overview} summarizes the framework, with detailed derivations
provided in Appendices~\ref{app:vicinal_analysis}--\ref{app:distillation_proof}.


\begin{figure*}[t]
    \centering
    \includegraphics[
        width=\textwidth,
        trim=0cm 2cm 0cm 3.0cm,
        clip
    ]{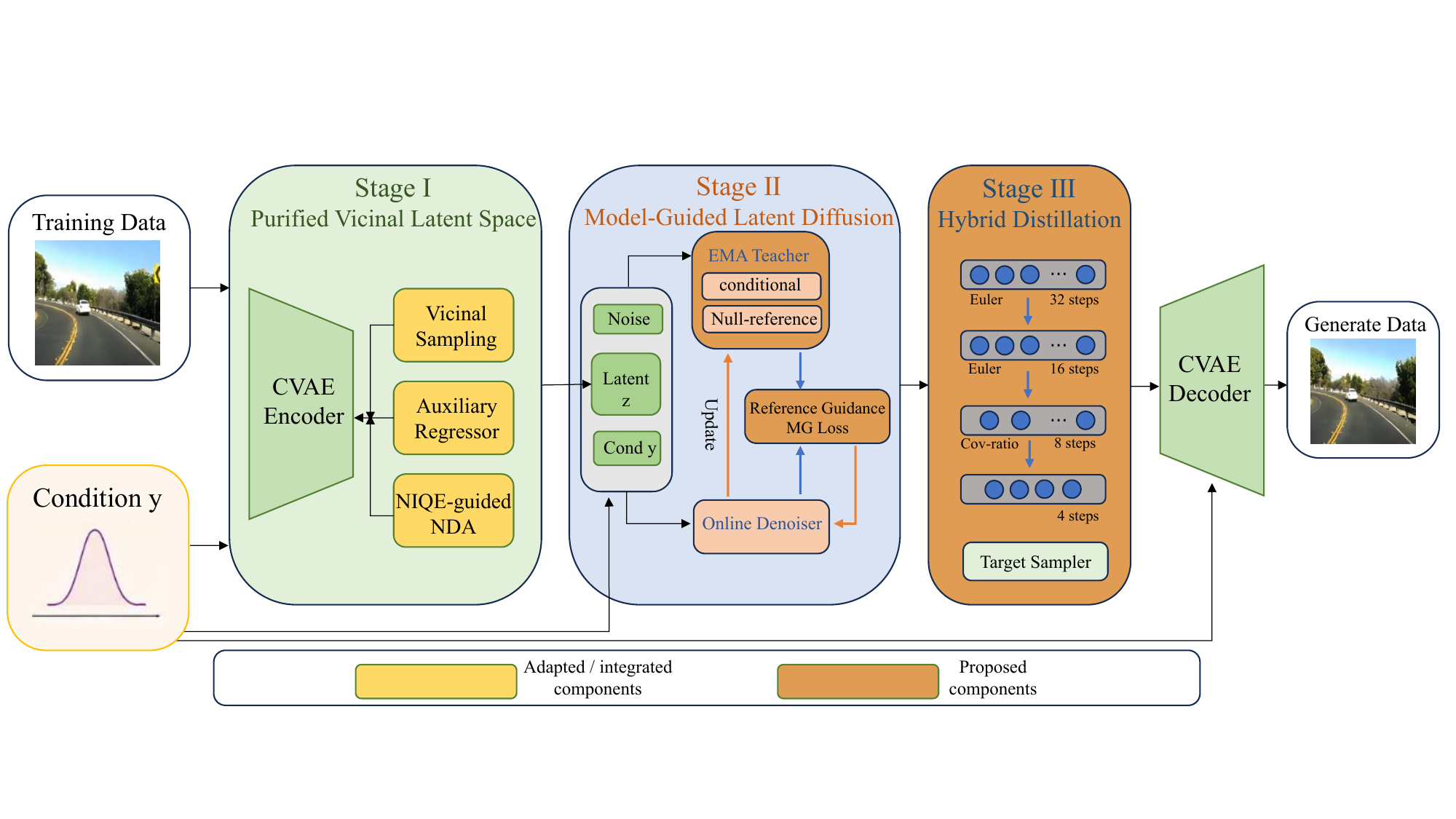}
    \caption{Overview of CAViLD. Yellow blocks denote adapted or
integrated components, while orange blocks denote the proposed
components. Unhighlighted blocks represent standard pipeline
elements.}
    \label{fig:overview}
\end{figure*}

\subsection{A Purified Continuous Conditional Latent Space}
\label{subsec:cvae}

Let $\mathcal{S}=\{(x_i,y_i)\}_{i=1}^{N_{\rm data}}$, where
$x_i\in\mathbb{R}^{C\times H\times W}$ and
$y_i\in\mathbb{R}$. A conditional encoder
$q_{\phi}(z\mid x,y)$ and decoder ${G}_{\psi}(z,y)$ are trained first
and then frozen. We use the reparameterization
$z=\mu_{\phi}(x,y)+s_{\phi}(x,y)\odot\xi$,
$\xi\sim\mathcal{N}(0,I)$, where
$s_{\phi}(x,y)=\exp\!\left(\tfrac12\log v_{\phi}(x,y)\right)$ is the posterior standard deviation. The decoder $G_\psi$ parameterizes the conditional
likelihood $p_\psi(x\mid z,y)$.
The base CVAE objective follows the standard conditional ELBO~\cite{sohn2015learning} formulation. However, conventional minibatch
training under an imbalanced continuous label distribution undersamples
sparse regions. We therefore use density-adaptive vicinal resampling to
modify the empirical frequency with which observed training pairs are
presented, without changing the per-sample ELBO form.

\textbf{Condition modulation.}
The scalar label is embedded as
$e_y=\operatorname{MLP}(y).$
After normalization in each residual block, Zero-Initialized
Decoupled Feature Modulation is applied as
\[
\widehat h_\ell
=
h_\ell\odot(1+\gamma_\ell)+\beta_\ell,
\qquad
\begin{bmatrix}
\gamma_\ell\\
\beta_\ell
\end{bmatrix}
=
W_\ell^m e_y+b_\ell^m.
\]
We initialize $W_\ell^m=0$ and $b_\ell^m=0$, such that
$\gamma_\ell=\beta_\ell=0$ at initialization and the modulation
reduces to an identity mapping. The label-dependent scale and
shift are learned during training, improving
optimization stability in sparse label regions.

\textbf{Adaptive vicinal resampling.}
Exact-label observations are rare for continuous conditions.
We therefore sample a base label uniformly from the set of distinct
observed labels, denoted by $\mathcal{Y}_{\rm train}$, and perturb it
with Gaussian noise:
\[
\begin{aligned}
&y_0
\sim
\operatorname{Unif}(\mathcal{Y}_{\rm train}),
\\
&\eta_{\rm KDE}
\sim
\mathcal{N}(0,\sigma_{\rm KDE}^{2}),
\qquad
y^\star
=
y_0+\eta_{\rm KDE}.
\end{aligned}
\]
$\sigma_{\rm KDE}$ is determined adaptively following
CcGAN-AVAR~\cite{ding2025imbalance}. Following the adaptive hard-support vicinity strategy, we construct a
symmetric neighborhood around each query label $y^\star$. For a
candidate radius $r \geq 0$, define
\[
\mathcal{N}_{r}(y^\star)
=
\left\{
i \in \{1,\ldots,N_{\rm data}\} :
|y_i-y^\star|\leq r
\right\},
\]
which contains the source examples whose observed labels lie within
distance $r$ of $y^\star$. Given a prescribed minimum neighborhood size
$N_{\rm AV}$, we choose the smallest radius that contains at least
$N_{\rm AV}$ source examples:
\begin{equation}
\kappa(y^\star)
=
\min
\left\{
r\geq0:
\left|
\mathcal{N}_{r}(y^\star)
\right|
\geq
N_{\rm AV}
\right\}.
\label{eq:adaptive_radius}
\end{equation}
The resulting adaptive neighborhood is then defined as
\[
\mathcal{N}(y^\star)
=
\mathcal{N}_{\kappa(y^\star)}(y^\star).
\]
Thus, the neighborhood contracts in densely sampled label regions and
expands in sparse regions so as to maintain a sufficient local source
population.
The corresponding hard vicinal indicator is
\begin{equation}
W_{i,y^\star}
=
\mathds{1}_{\{
|y_i-y^\star|
\leq
\kappa(y^\star)
\}}.
\label{eq:vicinal_weight}
\end{equation}
Since $W_{i,y^\star}$ is binary, we define
$\bar W_{i,y^\star}
=
W_{i,y^\star}/|\mathcal N(y^\star)|$
as the uniform sampling probability over the adaptive neighborhood.
Thus, source examples inside the neighborhood receive equal sampling
probability, whereas examples outside the neighborhood receive zero
probability.

The query label $y^\star$ is used only for neighborhood construction
and source resampling. The source pair retains its observed
label $y_i$, which is used as the condition for the CVAE:
\begin{equation}
z^{(i)}
\sim
q_\phi(z\mid x_i,y_i),
\qquad
\widehat{x}_i
=
G_\psi(z^{(i)},y_i).
\label{eq:cvae_source_pair}
\end{equation}

We use $\mathbb E_{y^\star}$ to denote expectation over the
query-label sampling procedure defined above. The reconstruction
and KL objectives are
\begin{align}
\mathcal L_{\rm rec}^{v}
&=
\mathbb E_{y^\star}
\sum_{i=1}^{N_{\rm data}}
\bar W_{i,y^\star}
\,
\mathbb E_{q_\phi(z\mid x_i,y_i)}
\Big[
\mathcal L_{\rm MSE}
\bigl(x_i,G_\psi(z,y_i)\bigr)
\nonumber\\
&\hspace{2.8cm}
+
\lambda_{\rm perc}
\mathcal L_{\rm perc}
\bigl(x_i,G_\psi(z,y_i)\bigr)
\Big],
\label{eq:vicinal_reconstruction}
\\
\mathcal L_{\rm KL}^{v}
&=
\mathbb E_{y^\star}
\sum_{i=1}^{N_{\rm data}}
\bar W_{i,y^\star}
\operatorname{KL}\!\left(
q_\phi(z\mid x_i,y_i)
\,\middle\|\,
\mathcal N(0,I)
\right).
\label{eq:vicinal_kl}
\end{align}

\begin{theorem}[Bias--variance trade-off]
\label{thm:vicinal_bound}
Under the assumptions detailed in
Appendix~\ref{app:vicinal_analysis}, the mean-squared error of the
vicinal estimator is bounded by the sum of a neighborhood-induced
bias term and an inverse-effective-sample-size variance term.
\end{theorem}

The bound favors local neighborhoods in dense regions and wider
label neighborhoods in sparse regions; detailed assumptions and proof
are given in Appendix~\ref{app:vicinal_analysis}.

\textbf{Semantic anchoring and negative purification.}
For each source pair selected by adaptive vicinal resampling,
a frozen regressor $R_{\omega_R}$ penalizes semantic deviation
between the reconstruction and its observed label:
\begin{equation}
\mathcal L_{\rm aux}
=
\mathbb E_{\substack{
y^\star,\;
i\sim\operatorname{Unif}(\mathcal N(y^\star))
}}
\left[
\left|
R_{\omega_R}(\widehat x_i)-y_i
\right|
-
\epsilon_{\rm aux}
\right]_{+},
\label{eq:auxiliary_regression}
\end{equation}
where $[u]_+=\max(u,0),$ $\epsilon_{\rm aux}$ is a tolerance margin and
$R_{\omega_R}$ remains frozen throughout CVAE training.

The negative replay buffer is constructed and periodically refreshed
during CVAE training using NIQE-based quality filtering. Specifically,
low-quality reconstructions identified by NIQE ranking are used to
update the negative buffer
\[
Q_{\rm bad}
=
\{(\tilde{x}_j,\tilde{y}_j)\}_{j=1}^{N_{\rm bad}}.
\]
During each discriminator update, a minibatch is uniformly sampled
from $Q_{\rm bad}$ and treated as additional negative examples.

To distinguish adversarial reweighting from the
indicator $W_{i,y^\star}$, we define the distance-aware confidence weight
\begin{equation}
\alpha_{i,y^\star}
=
\exp\!\left[
-\frac{(y_i-y^\star)^2}
{\kappa(y^\star)^2}
\right].
\label{eq:adversarial_weight}
\end{equation}
With the logistic losses
$
\ell_1(s)=-\log\rho(s),
\qquad
\ell_0(s)=-\log[1-\rho(s)],
$
where $\rho(s)=1/(1+\exp(-s))$ denotes the sigmoid function.
Hereafter, $\mathbb E_{y^\star,i}$ denotes expectation over the
query-label sampling procedure and
$i\sim\operatorname{Unif}(\mathcal N(y^\star))$.
The discriminator objectives are
\begin{align}
\mathcal L_D^{v}
&=
\mathbb E_{y^\star,i}
\left[
\alpha_{i,y^\star}
\left\{
\ell_1\bigl(D_\omega(x_i,y_i)\bigr)
+
\ell_0\bigl(
D_\omega(\operatorname{sg}[\widehat x_i],y_i)
\bigr)
\right\}
\right],
\label{eq:vicinal_discriminator}
\\
\mathcal L_D
&=
\mathcal L_D^{v}
+
\lambda_{\rm nda}
\mathbb E_{(\widetilde x,\widetilde y)\sim Q_{\rm bad}}
\left[
\ell_0\bigl(
D_\omega(\widetilde x,\widetilde y)
\bigr)
\right].
\label{eq:nda_discriminator}
\end{align}
The CVAE objective is
\begin{equation}
\begin{aligned}
\mathcal L_{\rm CVAE}
={}&
\mathcal L_{\rm rec}^{v}
+
\lambda_{\rm KL}\mathcal L_{\rm KL}^{v}
+
\lambda_{\rm aux}\mathcal L_{\rm aux}
\\
&+
\lambda_{\rm adv}
\mathbb E_{y^\star,i}
\left[
\alpha_{i,y^\star}
\ell_1\!\left(
D_\omega(\widehat x_i,y_i)
\right)
\right].
\end{aligned}
\label{eq:cvae_total}
\end{equation}
During alternating optimization, $D_\omega$ is updated using
Eq.~\eqref{eq:nda_discriminator}, whereas $(\phi,\psi)$ are updated
using Eq.~\eqref{eq:cvae_total} with $D_\omega$ fixed.
The auxiliary regressor $R_{\omega_R}$ remains frozen.
Since no generator-side NDA repulsion is used, $Q_{\rm bad}$
affects the CVAE only indirectly through the adversarial objective.

\subsection{Matrix-Form Model-Guided Latent Diffusion}
\label{subsec:ldm_iccdm}

\textbf{Condition-dependent corruption.}
Following iCCDM~\cite{ding2026iccdm}, we adopt its
condition-dependent covariance embedding in the latent space.
The embedding network $h_\nu$ is trained using the procedure of
iCCDM, with output dimension adjusted to the latent
dimension $d_z=C_zH_zW_z$, and is frozen during diffusion training.
We define the positive covariance embedding as
\[
a_\nu(y)
=
\exp\!\left[-h_\nu(y)\right]
\in\mathbb R_{>0}^{d_z},
\]
and use the diagonal covariance
\begin{equation}
K_\sigma(y)
=
\sigma^2 I
+
\lambda_{\rm cov}\sigma
\operatorname{Diag}\!\left(a_\nu(y)\right).
\label{eq:latent_covariance}
\end{equation}

The Stage-II training pairs follow the adaptive vicinal resampling
defined in Section~\ref{subsec:cvae}. Specifically, a query label
$y^\star$ is generated as described previously, and a source index
is sampled uniformly from its adaptive neighborhood:
\[
i
\sim
\operatorname{Unif}\!\left(\mathcal N(y^\star)\right).
\]
The selected source image $x_i$ is encoded by the frozen CVAE using
its observed label $y_i$. For diffusion modeling, we rescale the
sampled CVAE latent by a scalar $c>0$ to approximately normalize its
empirical scale, define the scaled clean latent as
\begin{equation}
z_0
=
c\left[
\mu_\phi(x_i,y_i)
+
s_\phi(x_i,y_i)\odot\xi
\right],
\qquad
\xi\sim\mathcal N(0,I).
\label{eq:clean_latent}
\end{equation}

Here, the observed label $y_i$ is used for source encoding, whereas
the query label $y^\star$ determines the diffusion condition and the
label-dependent corruption covariance. Hereafter, $y$ denotes the
Stage-II query label $y^\star$, while $y_i$ denotes the observed label
of the sampled source.

The above sampling procedure induces the vicinal clean-latent
distribution
\[
p_0^v(z_0\mid y)
=
\sum_{i=1}^{N_{\rm data}}
\bar W_{i,y}\,
q_{\phi,c}(z_0\mid x_i,y_i),
\]
where $q_{\phi,c}$ denotes the distribution of the scaled latent
$z_0=cz$ induced by $q_\phi(z\mid x_i,y_i)$.

The forward perturbation is
\begin{equation}
z_\sigma
=
z_0
+
K_\sigma(y)^{1/2}\epsilon,
\qquad
\epsilon\sim\mathcal N(0,I).
\label{eq:forward_perturbation}
\end{equation}

\textbf{Matrix preconditioning.}
We further adopt the matrix preconditioning of
iCCDM~\cite{ding2026iccdm}, replacing the image-space covariance
with the latent covariance in Eq.~\eqref{eq:latent_covariance}.
For matrix preconditioning, we adopt $S_z=s_z^2I$ with $s_z=0.5$
for the scaled latent representation. The denoiser is parameterized as
\begin{equation}
D_\theta(z_\sigma,\sigma,y)
=
C_{\rm skip}z_\sigma
+
C_{\rm out}
F_\theta
\left(
C_{\rm in}z_\sigma,
\tfrac14\log\sigma,
e_y
\right),
\label{eq:matrix_preconditioning}
\end{equation}
where
\begin{align}
\begin{aligned}
C_{\rm in}
&=\bigl(s_z^2I+K_\sigma(y)\bigr)^{-1/2},\\
C_{\rm skip}
&=s_z^2\bigl(s_z^2I+K_\sigma(y)\bigr)^{-1},\\
C_{\rm out}
&=s_zK_\sigma(y)^{1/2}
  \bigl(s_z^2I+K_\sigma(y)\bigr)^{-1/2}.
\end{aligned}
\label{eq:preconditioning_coefficients}
\end{align}
For $\sigma>0$, the preconditioning precision is
\begin{equation}
W_{\sigma,y}
=C_{\rm out}^{-2}
=s_z^{-2}I+K_\sigma(y)^{-1}.
\label{eq:precision_weight}
\end{equation}
Matrix operations are element-wise because
$K_\sigma(y)$ is diagonal. For the null prediction, $F_\theta$ receives the normalized noisy
latent, noise level, and null embedding $e_\varnothing$.
The requested geometry is retained through external preconditioning;
neither $y$ nor $K_\sigma(y)$ is a separate input to $F_\theta$.

\textbf{Internalizing model guidance.}
Condition dropout with probability $p_{\rm drop}$ is used to jointly
train conditional and null-conditioned predictions. To construct the
model-guidance target, the EMA teacher is evaluated twice on the same
preconstructed noisy latent $z_\sigma$, at the same noise level
$\sigma$ and requested label $y$: once through the conditional branch
and once through the null-conditioning mechanism. We denote the latter
prediction by
$
D_{\bar\theta}^{(y)}(z_\sigma,\sigma,\emptyset).
$
Here $\emptyset$ indicates that the explicit condition is suppressed,
whereas the superscript $(y)$ indicates that the reference prediction
is associated with the same requested label and noisy state as the
conditional prediction.
For a fixed requested label $y$, define
\begin{equation}
\begin{aligned}
 D_c
 &=
 D_{\bar\theta}(z_{\sigma},\sigma,y),\\
 D_{\emptyset}^{(y)}
 &=
 D_{\bar\theta}^{(y)}(z_{\sigma},\sigma,\emptyset),\\
 \Delta_{\bar\theta}^{\rm ref}
 &=
 D_c-D_{\emptyset}^{(y)}.
\end{aligned}
\label{eq:mg_score}
\end{equation}

Following the operational two-branch guidance rule used by matrix-form
continuous conditional diffusion, we define
\begin{equation}
 D^{\rm guide}
 =
 D_c+w_{\rm MG}\Delta_{\bar\theta}^{\rm ref}
 =
 (1+w_{\rm MG})D_c-w_{\rm MG}D_{\emptyset}^{(y)}.
 \label{eq:mg_clean_prediction}
\end{equation}
Under the CFG parameterization
$D_{\rm CFG}=D_{\emptyset}+\gamma(D_c-D_{\emptyset})$,
$\gamma=1+w_{\rm MG}$ matches the reference two-branch
coefficients, without implying exact CFG equivalence of the
trained single branch. For label-independent corruption and a
globally label-marginal null denoiser,
Eq.~\eqref{eq:mg_clean_prediction} recovers the standard CFG
clean-prediction identity. Under label-dependent corruption,
$D_{\emptyset}^{(y)}$ instead serves as a null-conditioned reference
under the fixed geometry $K_{\sigma}(y)$, rather than the exact
global-marginal denoiser. We train the conditional branch to absorb
this reference guidance direction.
Let $\bar\theta$ be an exponential moving average (EMA) of $\theta$.
With this evolving teacher, $w_{\rm MG}$ is an empirical training
coefficient. After warm-up, we construct the target
\begin{equation}
 z_0^{\rm MG}
 =
 z_0+w_{\rm MG}
 \operatorname{sg}\!\left[
 D_{\bar\theta}(z_{\sigma},\sigma,y)
 -
 D_{\bar\theta}^{(y)}(z_{\sigma},\sigma,\emptyset)
 \right],
 \label{eq:mg_target}
\end{equation}
and set $z_0^{\rm MG}=z_0$ during warm-up. The clean sample anchors the target,
the EMA difference supplies a reference-guidance direction, and stop‑gradient blocks gradient flow through the EMA teacher for every online‑model update.

For a positive-definite matrix $W$, we use the weighted norm
$\|u\|_{W}^{2}=u^{\mathsf T}Wu$. Let $\mathbb E_c$ and $\mathbb E_\emptyset$ denote expectations
under the corruption in Eq.~\eqref{eq:forward_perturbation}
and $z_\sigma=z_0+\sigma\epsilon$, respectively, including
the remaining Stage-II sampling.
Both branches retain the requested preconditioning.
The training objective is
\begin{equation}
\begin{aligned}
\mathcal L_{\rm diff}
={}&(1-p_{\rm drop})\mathbb E_c
\Bigl[
\bigl\|
D_\theta(z_\sigma,\sigma,y)-z_0^{\rm MG}
\bigr\|_{W_{\sigma,y}^{1/2}}^2
\Bigr]\\
&+p_{\rm drop}\mathbb E_\emptyset
\Bigl[
(s_z^{-2}+\sigma^{-2})^{1/2}\\
&\qquad\qquad\cdot
\bigl\|
D_\theta^{(y)}(z_\sigma,\sigma,\emptyset)
-z_0^{\rm MG}
\bigr\|_2^2
\Bigr].
\end{aligned}
\label{eq:model_guided_loss}
\end{equation}




The model-guidance
coefficient is activated after warm-up according to
\begin{equation}
w_{\rm MG}(s)
=
\begin{cases}
0, & s\leq s_{\rm warm},\\
w_{\rm MG}, & s>s_{\rm warm},
\end{cases}
\label{eq:mg_warmup}
\end{equation}
where $s$ denotes the optimization step. Before $s_{\rm warm}$,
the model is trained with the clean target $z_0$; afterwards, the
EMA prediction difference is incorporated into
$z_0^{\rm MG}$. Warm-up delays reference guidance until after initial
clean-target training.

Eq.~\eqref{eq:model_guided_loss} trains both online branches
toward the same model-guided target using their respective inputs.
At inference, only the conditional prediction is retained.

\textbf{Single-branch model-guided sampling.}
For the corruption process in Eq.~\eqref{eq:forward_perturbation}, define the clean-latent posterior mean as
\[
D^{*}(z_{\sigma},\sigma,y)
=
\mathbb{E}[z_0\mid z_{\sigma},\sigma,y].
\]
The corresponding conditional PF-ODE
\cite{song2021scorebased,chen2023probability} is
\begin{equation}
 \frac{d z_{\sigma}}{d\sigma}
 =
 B_{\sigma,y}
 \bigl[
 z_{\sigma}-D^{*}(z_{\sigma},\sigma,y)
 \bigr],
 \quad
 B_{\sigma,y}
 =
 \tfrac12K'_{\sigma}(y)K_{\sigma}(y)^{-1}.
 \label{eq:latent_pfode}
\end{equation}
After training, let $D_{\bar\theta}^{\rm MG}$ denote the
conditional prediction of the EMA denoiser obtained after
model-guided training. Replacing $D^{*}$ by this prediction yields the guided sampling ODE
\begin{equation}
\begin{aligned}
\frac{d z_{\sigma}}{d\sigma}
&=
B_{\sigma,y}
\left[
z_{\sigma}
-D_{\bar{\theta}}^{\mathrm{MG}}(z_{\sigma},\sigma,y)
\right].
\end{aligned}
\label{eq:latent_guided_ode}
\end{equation}
Equation~\eqref{eq:latent_guided_ode} does not claim exact preservation of the
original conditional marginals after guidance.
For $K_{\sigma}=\sigma^2I$, $B_{\sigma,y}=I/\sigma$, recovering the standard
EDM ODE. At a sufficiently large $\sigma_{\max}$, the corruption covariance
dominates the clean-latent covariance, motivating the boundary approximation
\[
 p_{\sigma_{\max}}(z\mid y)
 \approx
 \mathcal{N}\!\left(0,K_{\sigma_{\max}}(y)\right).
\]
We therefore initialize
$z_{\sigma_{\max}}\sim
\mathcal{N}(0,K_{\sigma_{\max}}(y))$ and integrate
Eq.~\eqref{eq:latent_guided_ode} backward toward the clean boundary
$\sigma=0$, without evaluating the vector field at the singular endpoint.
The resulting latent is decoded as
$\widehat x=G_{\psi}(z_0/c,y)$.

Each vector-field evaluation uses exactly one conditional denoiser call.
For an $M$-evaluation solver, standard two-branch CFG requires approximately
$2M$ denoiser evaluations, whereas the model-guided field requires $M$.
This saving is orthogonal to latent-space compression: the CVAE reduces the
cost of each evaluation, while model guidance removes the additional
null-conditioned forward pass at each integration point.

\subsection{Hybrid Step-Halving Distillation}
\label{subsec:distillation}

Model guidance removes the unconditional inference pass, but a conventional
ODE solver still evaluates the conditional denoiser many times. We therefore
distill the already model-guided, single-branch teacher; no unconditional
prediction is reintroduced during distillation.

\paragraph{Euler-based distillation from 32 to 8 steps}
Let
\[
\sigma_0=\sigma_{\max}>
\sigma_1>\cdots>\sigma_N=0
\]
be the $N$-step teacher schedule, where $N\in\{32,16\}$ for the
two distillation stages. We sample an index
$k\in\{0,2,\ldots,N-2\}$ and distill the two teacher transitions
from $\sigma_k$ to $\sigma_{k+2}$ into one student transition.

We train on forward-corrupted states $z_{\sigma_k}$ constructed by
Eq.~\eqref{eq:forward_perturbation} at $\sigma=\sigma_k$, where
$z_0$ is obtained from the sampled source pair.
At each optimization step, an EMA $N$-step teacher, evaluated under
stop-gradient, applies two fine solver transitions to produce
\begin{equation}
z_{\sigma_{k+2}}^{T}
=
\Phi_{\bar\theta_N}
\left(
z_{\sigma_k};
\sigma_k\!\rightarrow\!\sigma_{k+1}
\!\rightarrow\!\sigma_{k+2},
y
\right).
\label{eq:teacher_endpoint}
\end{equation}

Let
\[
\Delta_k=\sigma_{k+2}-\sigma_k<0,
\qquad
B_k=B_{\sigma_k,y}.
\]
We invert one coarse Euler transition of
Eq.~\eqref{eq:latent_guided_ode} and obtain the endpoint-matching denoised-prediction target
\begin{equation}
\begin{aligned}
D_k^{\dagger}
&=
z_{\sigma_k}
-
B_k^{-1}
\frac{
z_{\sigma_{k+2}}^{T}-z_{\sigma_k}
}{
\Delta_k
}
\\
&=
z_{\sigma_k}
-
2K_{\sigma_k}(y)
\bigl(K'_{\sigma_k}(y)\bigr)^{-1}
\frac{
z_{\sigma_{k+2}}^{T}-z_{\sigma_k}
}{
\Delta_k
}.
\end{aligned}
\label{eq:distilled_clean_target}
\end{equation}

\begin{theorem}[Euler endpoint consistency]
\label{thm:endpoint_consistency}
For $\sigma_k>0$, the target $D_k^\dagger$ in
Eq.~\eqref{eq:distilled_clean_target} is the unique prediction whose
one-step Euler transition reaches the two-step teacher endpoint
$z_{\sigma_{k+2}}^T$.
\end{theorem}The proof is provided in Appendix~\ref{app:distillation_proof}.

The $N/2$-step student is initialized from the preceding-stage
checkpoint and minimizes
\begin{equation}
\begin{aligned}
\mathcal{L}_{\rm distill}
&\\
&\hspace{-2.3em}=
\mathbb{E}_{k,(z_0,y),\epsilon}
\Bigl\|
W_{\sigma_k,y}^{1/2}
\left(
D_{\theta_{N/2}}(z_{\sigma_k},\sigma_k,y)
-
\operatorname{sg}[D_k^{\dagger}]
\right)
\Bigr\|_2^2.
\end{aligned}
\label{eq:distillation_loss}
\end{equation}
The precision $W_{\sigma_k,y}$ rescales prediction residuals,
while endpoint consistency follows from $D_k^\dagger$.
We apply this Euler-based endpoint distillation successively as
$32\rightarrow16\rightarrow8$.

\paragraph{Final covariance-ratio distillation from 8 to 4 steps}
For the $8\rightarrow4$ stage, we use a covariance-ratio
transition on the four-step schedule. For two noise levels
$\sigma_a>\sigma_b$, define
\begin{equation}
R_{a\rightarrow b}(y)
=
\left(
K_{\sigma_b}(y)
K_{\sigma_a}(y)^{-1}
\right)^{1/2},
\label{eq:distillation_covariance_ratio}
\end{equation}
where the inverse and square root are evaluated elementwise because
$K_\sigma(y)$ is diagonal. The corresponding deterministic transition
is
\begin{equation}
\begin{aligned}
\Psi_{\vartheta}
\left(
z_{\sigma_a};
\sigma_a\!\rightarrow\!\sigma_b,y
\right)
&\\
&\hspace{-8em}=
D_{\vartheta}(z_{\sigma_a},\sigma_a,y)
+
R_{a\rightarrow b}(y)
\left[
z_{\sigma_a}
-
D_{\vartheta}(z_{\sigma_a},\sigma_a,y)
\right].
\end{aligned}
\label{eq:covariance_ratio_transition}
\end{equation}
For final stage, we retain the trained 8-step denoiser but switch
its teacher rollout from Euler to the covariance-ratio rule:
\begin{equation}
z_{\sigma_{k+2}}^{T}
=
\Psi_{\bar\theta_8}
\left(
\Psi_{\bar\theta_8}
\left(
z_{\sigma_k};
\sigma_k\!\rightarrow\!\sigma_{k+1},y
\right);
\sigma_{k+1}\!\rightarrow\!\sigma_{k+2},y
\right).
\label{eq:ratio_teacher_endpoint}
\end{equation}

Let
\[
R_k
=
R_{k\rightarrow k+2}(y)
=
\left(
K_{\sigma_{k+2}}(y)
K_{\sigma_k}(y)^{-1}
\right)^{1/2}.
\]
Equating one coarse student transition with the two-step teacher
endpoint and solving for the denoised prediction gives
\begin{equation}
D_{k,\mathrm{ratio}}^{\dagger}
=
\left(I-R_k\right)^{-1}
\left(
z_{\sigma_{k+2}}^{T}
-
R_k z_{\sigma_k}
\right).
\label{eq:ratio_distillation_target}
\end{equation}
\begin{theorem}[Covariance-ratio endpoint consistency]
\label{thm:ratio_endpoint_consistency}
Under the monotone diagonal covariance schedule in Eq.~(5),
$I-R_k$ is nonsingular, and $D_{k,\mathrm{ratio}}^\dagger$ uniquely matches the
two-step covariance-ratio teacher endpoint in Eq.~\eqref{eq:ratio_teacher_endpoint}.
\end{theorem}
The 4-step student is trained with
Eq.~\eqref{eq:distillation_loss} after replacing
$D_k^\dagger$ by $D_{k,\mathrm{ratio}}^\dagger$.

The covariance-ratio transition is first-order consistent with the
guided latent ODE because
\begin{equation}
R_{\sigma\rightarrow\sigma+\Delta\sigma}(y)
=
I
+
\Delta\sigma B_{\sigma,y}
+
\mathcal{O}(\Delta\sigma^2).
\label{eq:ratio_local_consistency}
\end{equation}
It analytically accounts for the covariance scaling while holding the
denoised prediction fixed over each transition. At inference, the
4-step student uses the same coarse covariance-ratio update in
Eq.~\eqref{eq:covariance_ratio_transition}. Each transition evaluates
the single conditional denoiser exactly once; hence the final sampler requires four conditional NFEs,
followed by one CVAE decode.

The complete hybrid distillation schedule is therefore
\[
32
\xrightarrow{\text{Euler}}
16
\xrightarrow{\text{Euler}}
8
\xrightarrow{\text{covariance ratio}}
4.
\]
Appendix~\ref{app:distillation_proof} provides the endpoint-consistency
and local-consistency derivations. The overall training procedure is
summarized in Algorithms~\ref{alg:latent_training}--\ref{alg:progressive_distillation},
with the four-step sampling procedure described in
Appendix~\ref{app:algorithm}.

\section{Experiments}
\label{sec:experiments}

\begin{table*}[!t]
    \setlength{\tabcolsep}{1.2mm}
    \renewcommand{\arraystretch}{1.10}
    \small
    \centering

    \caption{Quantitative comparison of image generation performance under
    low-to-medium resolutions ($\leq 128\times128$). Arrows
    ($\uparrow/\downarrow$) present the preferred optimization directions.
    For the primary quality metric (SFID), the best results are highlighted
    in \textbf{bold} and the second-best in \underline{underlined}.
    Evaluation profiles corrupted by severe training pathologies
    (e.g., severe label inconsistency or mode collapse) are denoted in
    \fail{red}.}

    \resizebox{\textwidth}{!}{%
    \begin{tabular}{c:c:c:c:cccc:cc}
        \hline
        \shortstack{\textbf{Dataset}}
        &
        \shortstack{\textbf{Condition}}
        &
        \shortstack{\textbf{Framework}}
        &
        \textbf{Method}
        &
        \shortstack{\textbf{SFID}} $\downarrow$
        &
        \textbf{NIQE} $\downarrow$
        &
        \textbf{Diversity} $\uparrow$
        &
        \shortstack{\textbf{LS}} $\downarrow$
        &
        \textbf{img/sec} $\uparrow$
        &
        \shortstack{\textbf{Mem(GiB)}} $\downarrow$
        \\
        \hline

        \multirow{11}*{\shortstack{\textbf{RC-49}\\(64$\times$64)}}
        & \multirow{4}[0]{*}{Text}
        & \multirow{4}[0]{*}{Diffusion}
        & SD 3 Medium (ICML'24)~\cite{esser2024scaling}
        & \fail{1.215} & 1.996 & 2.272 & \fail{28.726} & 1.90 & 19.74 \\

        & & &
        FLUX.1 (arXiv'25)~\cite{labs2025flux1kontextflowmatching}
        & 0.515 & 2.330 & \fail{2.512} & \fail{12.751} & 3.87 & 42.08 \\

        & & &
        Qwen-Image (arXiv'25)~\cite{wu2025qwen}
        & \fail{1.132} & 2.436 & \fail{2.470} & \fail{27.757}
        & 9.58 & 56.91 \\

        & & &
        MIRO (ICML'26)~\cite{dufour2026miro}
        & \fail{2.956} & 3.146 & \fail{0.791} & \fail{10.66}
        & 1.19 & 6.53 \\

        \cdashline{2-10}

        & \multirow{7}[0]{*}{Continuous}
        & \multirow{3}[0]{*}{GAN}
        & CcGAN (T-PAMI'23)~\cite{ding2023ccgan}
        & 0.126 & 1.809 & 3.451 & 2.655 & 6243.06 & 2.62 \\

        & & &
        Dual-NDA (AAAI'24)~\cite{ding2024turning}
        & 0.148 & 1.808 & 3.344 & 2.211 & 6243.06 & 2.62 \\

        & & &
        CcGAN-AVAR (arXiv'25)~\cite{ding2025imbalance}
        & \textbf{0.042} & 1.734 & 3.723 & 1.270 & 4162.04 & 3.71 \\

        \cdashline{3-10}

        & &
        \multirow{4}[0]{*}{Diffusion}
        & CcDPM (AAAI'24)~\cite{zhao2024ccdpm}
        & \fail{0.970} & 2.153 & 3.581 & \fail{24.174} & 2.92 & 5.07 \\

        & & &
        CCDM (TMM'26)~\cite{ding2026ccdm}
        & 0.049 & 2.086 & 3.698 & 1.074 & 2.92 & 5.53 \\

        & & &
        iCCDM (arXiv'26)~\cite{ding2026iccdm}
        & 0.049 & 1.718 & 3.696 & 0.813 & 12.67 & 3.69 \\

        & & &
        \textbf{CAViLD (Ours)}
        & \underline{0.048} & 2.078 & 3.706 & 1.068 & 370.72 & 2.46 \\

        \hline

        \multirow{11}*{\shortstack{\textbf{UTKFace}\\(64$\times$64)}}
        & \multirow{4}[0]{*}{Text}
        & \multirow{4}[0]{*}{Diffusion}
        & SD 3 Medium (ICML'24)~\cite{esser2024scaling}
        & \fail{2.012} & 2.068 & \fail{0.426} & 5.452 & 1.50 & 19.74 \\

        & & &
        FLUX.1 (arXiv'25)~\cite{labs2025flux1kontextflowmatching}
        & 0.816 & 1.915 & \fail{0.713} & 9.174 & 3.80 & 42.08 \\

        & & &
        Qwen-Image (arXiv'25)~\cite{wu2025qwen}
        & 1.973 & 2.135 & 1.051 & \fail{10.770} & 9.37 & 56.90 \\

        & & &
        MIRO (ICML'26)~\cite{dufour2026miro}
        & \fail{4.763} & 2.193 & 1.108 & \fail{10.888} & 1.19 & 6.28 \\

        \cdashline{2-10}

        & \multirow{7}[0]{*}{Continuous}
        & \multirow{3}[0]{*}{GAN}
        & CcGAN (T-PAMI'23)~\cite{ding2023ccgan}
        & 0.413 & 1.733 & 1.329 & 8.240 & 5555.56 & 2.44 \\

        & & &
        Dual-NDA (AAAI'24)~\cite{ding2024turning}
        & 0.396 & 1.678 & 1.298 & 6.765 & 5555.56 & 2.44 \\

        & & &
        CcGAN-AVAR (arXiv'25)~\cite{ding2025imbalance}
        & 0.356 & 1.691 & 1.278 & 6.696 & 4166.67 & 3.71 \\

        \cdashline{3-10}

        & &
        \multirow{4}[0]{*}{Diffusion}
        & CcDPM (AAAI'24)~\cite{zhao2024ccdpm}
        & 0.466 & 1.560 & 1.211 & 6.868 & 2.10 & 5.70 \\

        & & &
        CCDM (TMM'26)~\cite{ding2026ccdm}
        & 0.363 & 1.542 & 1.184 & 6.164 & 2.10 & 6.14 \\

        & & &
        iCCDM (arXiv'26)~\cite{ding2026iccdm}
        & \underline{0.353} & 1.561 & 1.146 & 5.788 & 11.13 & 4.94 \\

        & & &
        \textbf{CAViLD (Ours)}
        & \textbf{0.344} & 1.611 & 1.206 & 5.608 & 355.48 & 2.53 \\

        \hline

        \multirow{11}*{\shortstack{\textbf{UTKFace}\\(128$\times$128)}}
        & \multirow{4}[0]{*}{Text}
        & \multirow{4}[0]{*}{Diffusion}
        & SD 3 Medium (ICML'24)~\cite{esser2024scaling}
        & \fail{1.582} & 1.901 & \fail{0.396} & 5.019 & 1.43 & 19.74 \\

        & & &
        FLUX.1 (arXiv'25)~\cite{labs2025flux1kontextflowmatching}
        & 0.643 & 1.748 & \fail{0.661} & 9.281 & 3.95 & 42.08 \\

        & & &
        Qwen-Image (arXiv'25)~\cite{wu2025qwen}
        & 1.517 & 1.653 & 0.971 & 9.731 & 9.73 & 59.33 \\

        & & &
        MIRO (ICML'26)~\cite{dufour2026miro}
        & \fail{3.590} & 2.069 & \fail{0.691} & \fail{12.001}
        & 1.20 & 6.28 \\

        \cdashline{2-10}

        & \multirow{7}[0]{*}{Continuous}
        & \multirow{3}[0]{*}{GAN}
        & CcGAN (T-PAMI'23)~\cite{ding2023ccgan}
        & 0.367 & 1.113 & 1.199 & 7.747 & 1282.05 & 7.58 \\

        & & &
        Dual-NDA (AAAI'24)~\cite{ding2024turning}
        & 0.361 & 1.081 & 1.257 & 6.310 & 1282.05 & 7.58 \\

        & & &
        CcGAN-AVAR (arXiv'25)~\cite{ding2025imbalance}
        & 0.297 & 1.173 & 1.251 & 6.586 & 1190.48 & 11.57 \\

        \cdashline{3-10}

        & &
        \multirow{4}[0]{*}{Diffusion}
        & CcDPM (AAAI'24)~\cite{zhao2024ccdpm}
        & 0.529 & 1.114 & 1.195 & 7.933 & 1.12 & 17.91 \\

        & & &
        CCDM (TMM'26)~\cite{ding2026ccdm}
        & 0.319 & 1.077 & 1.178 & 6.359 & 1.12 & 19.52 \\

        & & &
        iCCDM (arXiv'26)~\cite{ding2026iccdm}
        & \underline{0.294} & 1.048 & 1.225 & 5.823 & 2.85 & 5.04 \\

        & & &
        \textbf{CAViLD (Ours)}
        & \textbf{0.284} & 1.088 & 1.131 & 5.869 & 25.72 & 5.10 \\

        \hline

        \multirow{11}*{\shortstack{\textbf{Steering Angle}\\(64$\times$64)}}
        & \multirow{4}[0]{*}{Text}
        & \multirow{4}[0]{*}{Diffusion}
        & SD 3 Medium (ICML'24)~\cite{esser2024scaling}
        & \fail{4.007} & 1.685 & 0.833 & \fail{19.885} & 1.42 & 19.74 \\

        & & &
        FLUX.1 (arXiv'25)~\cite{labs2025flux1kontextflowmatching}
        & 1.795 & 1.766 & 1.416 & \fail{19.911} & 3.83 & 42.08 \\

        & & &
        Qwen-Image (arXiv'25)~\cite{wu2025qwen}
        & 2.557 & 2.259 & 1.357 & \fail{20.734} & 9.56 & 56.9 \\

        & & &
        MIRO (ICML'26)~\cite{dufour2026miro}
        & \fail{3.497} & \fail{2.689} & 0.962 & \fail{32.621}
        & 1.18 & 6.53 \\

        \cdashline{2-10}

        & \multirow{7}[0]{*}{Continuous}
        & \multirow{3}[0]{*}{GAN}
        & CcGAN (T-PAMI'23)~\cite{ding2023ccgan}
        & 1.334 & 1.784 & 1.234 & 14.807 & 3086.42 & 2.50 \\

        & & &
        Dual-NDA (AAAI'24)~\cite{ding2024turning}
        & 1.114 & 1.738 & 1.251 & 11.809 & 3086.42 & 2.50 \\

        & & &
        CcGAN-AVAR (arXiv'25)~\cite{ding2025imbalance}
        & 0.809 & 1.800 & 1.204 & 6.963 & 1157.41 & 3.86 \\

        \cdashline{3-10}

        & &
        \multirow{4}[0]{*}{Diffusion}
        & CcDPM (AAAI'24)~\cite{zhao2024ccdpm}
        & 0.939 & 1.761 & 1.150 & 10.999 & 3.33 & 2.28 \\

        & & &
        CCDM (TMM'26)~\cite{ding2026ccdm}
        & 0.742 & 1.778 & 1.088 & 5.823 & 3.33 & 2.74 \\

        & & &
        iCCDM (arXiv'26)~\cite{ding2026iccdm}
        & \underline{0.501} & 1.700 & 1.093 & 4.891 & 12.82 & 4.94 \\

        & & &
        \textbf{CAViLD (Ours)}
        & \textbf{0.358} & 1.724 & 1.048 & 5.596 & 609.76 & 1.46 \\

        \hline

        \multirow{11}*{\shortstack{\textbf{Steering Angle}\\(128$\times$128)}}
        & \multirow{4}[0]{*}{Text}
        & \multirow{4}[0]{*}{Diffusion}
        & SD 3 Medium (ICML'24)~\cite{esser2024scaling}
        & \fail{4.191} & 4.071 & \fail{0.512} & \fail{20.072}
        & 1.48 & 19.74 \\

        & & &
        FLUX.1 (arXiv'25)~\cite{labs2025flux1kontextflowmatching}
        & 1.676 & 3.090 & 1.072 & \fail{20.526} & 3.86 & 42.08 \\

        & & &
        Qwen-Image (arXiv'25)~\cite{wu2025qwen}
        & 2.910 & 3.749 & 1.062 & \fail{18.455} & 9.68 & 59.34 \\

        & & &
        MIRO (ICML'26)~\cite{dufour2026miro}
        & \fail{3.869} & 4.784 & 0.977 & \fail{29.707}
        & 1.14 & 6.53 \\

        \cdashline{2-10}

        & \multirow{7}[0]{*}{Continuous}
        & \multirow{3}[0]{*}{GAN}
        & CcGAN (T-PAMI'23)~\cite{ding2023ccgan}
        & 1.689 & 2.411 & 1.088 & \fail{18.438} & 1543.21 & 7.58 \\

        & & &
        Dual-NDA (AAAI'24)~\cite{ding2024turning}
        & 1.390 & 2.135 & 1.133 & 14.099 & 1543.21 & 7.58 \\

        & & &
        CcGAN-AVAR (arXiv'25)~\cite{ding2025imbalance}
        & 0.888 & 2.288 & 1.123 & 7.507 & 1207.73 & 11.57 \\

        \cdashline{3-10}

        & &
        \multirow{4}[0]{*}{Diffusion}
        & CcDPM (AAAI'24)~\cite{zhao2024ccdpm}
        & 1.285 & 1.989 & 1.203 & \fail{18.325} & 1.23 & 5.27 \\

        & & &
        CCDM (TMM'26)~\cite{ding2026ccdm}
        & 0.987 & 1.977 & 1.118 & 11.829 & 1.23 & 6.83 \\

        & & &
        iCCDM (arXiv'26)~\cite{ding2026iccdm}
        & \underline{0.742} & 1.973 & 1.007 & 5.652 & 2.75 & 5.04 \\

        & & &
        \textbf{CAViLD (Ours)}
        & \textbf{0.316} & 2.584 & 0.969 & 5.170 & 42.65 & 5.16 \\

        \hline
    \end{tabular}%
    }

    \label{tab:low_res_comparison}
\end{table*}

\begin{table*}[!t]

\setlength{\tabcolsep}{1.2mm}
\renewcommand{\arraystretch}{1.10}

\small
\centering

\caption{Average quality and efficiency of high-resolution generated images
($\geq 192\times192$). Arrows ($\uparrow/\downarrow$) indicate the
preferred optimization directions. For the primary quality metric
(SFID), the best results are highlighted in \textbf{bold} and the
second-best in \underline{underlined}. Evaluation profiles corrupted
by severe training pathologies (e.g., severe label inconsistency or
mode collapse) are denoted in \fail{red}.}
\label{tab:high_res_comparison}

\resizebox{\textwidth}{!}{%

\begin{tabular}{c:c:c:c:cccc:cc}

\hline

\makecell{\textbf{Dataset}}
&
\makecell{\textbf{Condition}}
&
\makecell{\textbf{Framework}}
&
\textbf{Method}
&
\makecell{\textbf{SFID} $\downarrow$}
&
\textbf{NIQE} $\downarrow$
&
\textbf{Diversity} $\uparrow$
&
\makecell{\textbf{LS} $\downarrow$}
&
\makecell{\textbf{img/sec} $\uparrow$}
&
\makecell{\textbf{Mem(GiB)} $\downarrow$}
\\

\hline


\multirow{11}{*}{
    \makecell{
        \textbf{UTKFace}\\
        (192$\times$192)
    }
}
&
\multirow{4}{*}{Text}
&
\multirow{4}{*}{Diffusion}
&
SD 3 Medium (ICML'24)~\cite{esser2024scaling}
&
\fail{1.682}
&
2.436
&
\fail{0.383}
&
5.188
&
1.48
&
20.98
\\

& & &
FLUX.1 (arXiv'25)~\cite{labs2025flux1kontextflowmatching}
&
0.771
&
2.499
&
0.602
&
9.191
&
3.96
&
42.08
\\

& & &
Qwen-Image (arXiv'25)~\cite{wu2025qwen}
&
1.617
&
1.909
&
0.943
&
9.868
&
9.64
&
62.43
\\

& & &
MIRO (ICML'26)~\cite{dufour2026miro}
&
\fail{3.774}
&
2.197
&
1.036
&
\fail{11.026}
&
1.19
&
6.28
\\

\cdashline{2-10}

&
\multirow{7}{*}{Continuous}
&
\multirow{3}{*}{GAN}
&
CcGAN (T-PAMI'23)~\cite{ding2023ccgan}
&
0.499
&
1.661
&
1.207
&
7.885
&
617.28
&
18.18
\\

& & &
Dual-NDA (AAAI'24)~\cite{ding2024turning}
&
0.487
&
1.483
&
1.201
&
6.730
&
617.28
&
18.18
\\

& & &
CcGAN-AVAR (arXiv'25)~\cite{ding2025imbalance}
&
0.435
&
1.588
&
1.187
&
6.377
&
1111.11
&
13.18
\\

\cdashline{3-10}

& &
\multirow{4}{*}{Diffusion}
&
CcDPM (AAAI'24)~\cite{zhao2024ccdpm}
&
0.970
&
1.522
&
1.187
&
\fail{11.224}
&
0.80
&
19.93
\\

& & &
CCDM (TMM'26)~\cite{ding2026ccdm}
&
0.467
&
1.242
&
1.148
&
7.336
&
0.80
&
22.99
\\

& & &
iCCDM (arXiv'26)~\cite{ding2026iccdm}
&
\underline{0.433}
&
1.850
&
1.210
&
6.323
&
1.25
&
10.87
\\

& & &
\textbf{CAViLD (Ours)}
&
\textbf{0.424}
&
1.892
&
1.193
&
5.664
&
26.57
&
10.63
\\

\hline


\multirow{11}{*}{
    \makecell{
        \textbf{UTKFace}\\
        (256$\times$256)
    }
}
&
\multirow{4}{*}{Text}
&
\multirow{4}{*}{Diffusion}
&
SD 3 Medium (ICML'24)~\cite{esser2024scaling}
&
\fail{0.906}
&
1.524
&
\fail{0.454}
&
5.850
&
1.45
&
24.23
\\

& & &
FLUX.1 (arXiv'25)~\cite{labs2025flux1kontextflowmatching}
&
0.369
&
1.576
&
\fail{0.614}
&
9.010
&
3.93
&
42.08
\\

& & &
Qwen-Image (arXiv'25)~\cite{wu2025qwen}
&
0.953
&
1.362
&
0.980
&
\fail{10.941}
&
9.82
&
65.8
\\

& & &
MIRO (ICML'26)~\cite{dufour2026miro}
&
\fail{1.848}
&
2.467
&
1.066
&
\fail{9.249}
&
1.15
&
6.53
\\

\cdashline{2-10}

&
\multirow{7}{*}{Continuous}
&
\multirow{3}{*}{GAN}
&
CcGAN (T-PAMI'23)~\cite{ding2023ccgan}
&
0.347
&
1.308
&
1.053
&
6.010
&
580.62
&
22.29
\\

& & &
Dual-NDA (AAAI'24)~\cite{ding2024turning}
&
0.319
&
1.381
&
1.216
&
5.712
&
580.62
&
22.29
\\

& & &
CcGAN-AVAR (arXiv'25)~\cite{ding2025imbalance}
&
\textbf{0.196}
&
1.347
&
1.219
&
5.798
&
609.43
&
17.32
\\

\cdashline{3-10}

& &
\multirow{4}{*}{Diffusion}
&
CcDPM (AAAI'24)~\cite{zhao2024ccdpm}
&
0.559
&
1.552
&
0.854
&
\fail{12.121}
&
0.30
&
15.67
\\

& & &
CCDM (TMM'26)~\cite{ding2026ccdm}
&
0.268
&
1.364
&
1.121
&
7.555
&
0.30
&
21.75
\\

& & &
iCCDM (arXiv'26)~\cite{ding2026iccdm}
&
\underline{0.218}
&
1.207
&
1.142
&
6.199
&
0.69
&
18.35
\\

& & &
\textbf{CAViLD (Ours)}
&
\textbf{0.196}
&
1.289
&
1.105
&
5.778
&
23.57
&
16.22
\\

\hline


\multirow{11}{*}{
    \makecell{
        \textbf{Steering Angle}\\
        (256$\times$256)
    }
}
&
\multirow{4}{*}{Text}
&
\multirow{4}{*}{Diffusion}
&
SD 3 Medium (ICML'24)~\cite{esser2024scaling}
&
\fail{1.859}
&
1.595
&
0.907
&
\fail{19.806}
&
1.51
&
24.23
\\

& & &
FLUX.1 (arXiv'25)~\cite{labs2025flux1kontextflowmatching}
&
0.945
&
2.353
&
1.078
&
\fail{17.228}
&
3.93
&
42.08
\\

& & &
Qwen-Image (arXiv'25)~\cite{wu2025qwen}
&
1.383
&
2.525
&
1.001
&
\fail{16.034}
&
9.82
&
65.82
\\

& & &
MIRO (ICML'26)~\cite{dufour2026miro}
&
\fail{1.965}
&
1.804
&
0.802
&
\fail{27.598}
&
1.15
&
6.53
\\

\cdashline{2-10}

&
\multirow{7}{*}{Continuous}
&
\multirow{3}{*}{GAN}
&
CcGAN (T-PAMI'23)~\cite{ding2023ccgan}
&
0.984
&
1.999
&
1.054
&
8.399
&
539.16
&
13.63
\\

& & &
Dual-NDA (AAAI'24)~\cite{ding2024turning}
&
0.967
&
1.834
&
0.951
&
8.338
&
539.16
&
13.63
\\

& & &
CcGAN-AVAR (arXiv'25)~\cite{ding2025imbalance}
&
0.683
&
1.966
&
1.113
&
4.958
&
620.54
&
17.32
\\

\cdashline{3-10}

& &
\multirow{4}{*}{Diffusion}
&
CcDPM (AAAI'24)~\cite{zhao2024ccdpm}
&
1.217
&
1.748
&
1.330
&
\fail{29.239}
&
0.30
&
16.30
\\

& & &
CCDM (TMM'26)~\cite{ding2026ccdm}
&
0.902
&
1.616
&
1.151
&
\fail{23.057}
&
0.30
&
21.75
\\

& & &
iCCDM (arXiv'26)~\cite{ding2026iccdm}
&
\underline{0.582}
&
1.390
&
0.981
&
4.743
&
0.69
&
18.35
\\

& & &
\textbf{CAViLD (Ours)}
&
\textbf{0.371}
&
1.519
&
0.978
&
3.450
&
23.79
&
16.22
\\

\hline

\end{tabular}%
}

\end{table*}

To comprehensively evaluate the effectiveness and efficiency of our proposed
CAViLD framework, we conduct extensive experiments on standard continuous
conditional generative modeling benchmarks. We investigate three aspects:
generation fidelity and continuous label consistency, sampling efficiency,
and the quality of the learned latent representation under imbalanced
continuous conditions.

\subsection{Experimental Setup}
\label{subsec:exp_setup}

{\setlength{\parindent}{0cm}\textbf{Datasets.}}
We evaluate our proposed CAViLD on three widely used continuous conditional generative modeling benchmarks following previous CCGM protocols~\cite{ding2023ccgan,ding2026iccdm}.
These datasets cover diverse continuous control scenarios, including object pose manipulation, age-conditioned face synthesis, and autonomous driving, with image resolutions ranging from $64\times64$ to $256\times256$.

\begin{figure}[!htbp]
    \centering

    \includegraphics[
        width=0.6\columnwidth,
        trim=4 3 4 3,
        clip
    ]{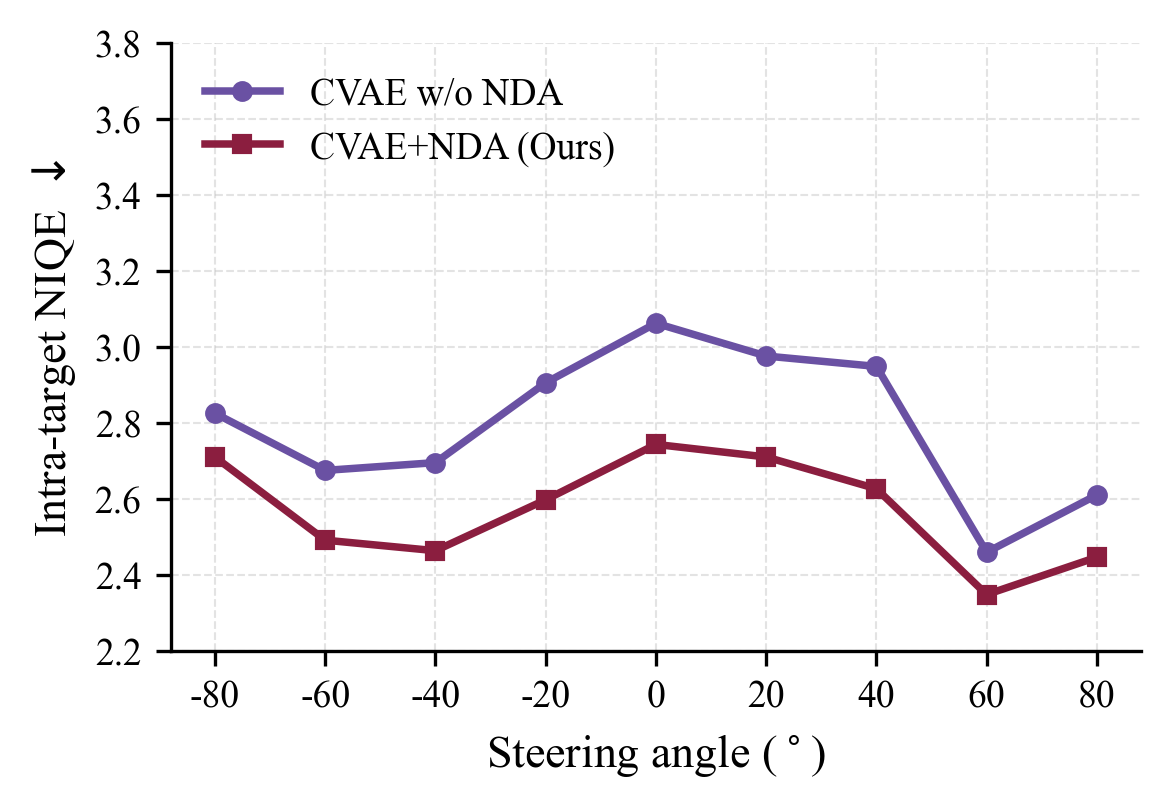}

    \caption{
    Effect of NDA on CVAE reconstruction quality. CVAE+NDA achieves lower NIQE across most steering-angle conditions.
    }
    \label{fig:nda_niqe_analysis}
\end{figure}

\begin{table}[!h]
    \centering

    \caption{Ablation of core components and acceleration strategies on the
    Steering Angle ($64\times64$) dataset. `NFEs' denotes the number of function
    evaluations, and `img/sec' denotes sampling throughput.}
    \label{tab:ablation_sa64}

    \setlength{\tabcolsep}{2.5pt}
    \renewcommand{\arraystretch}{1.10}
    \footnotesize

    \begin{tabular}{p{0.36\columnwidth}:c:ccc}

        \hline

        \textbf{Configurations}
        &
        \textbf{NFEs $\downarrow$}
        &
        \textbf{img/sec $\uparrow$}
        &
        \textbf{SFID $\downarrow$}
        &
        \textbf{LS $\downarrow$}
        \\

        \hline

        Baseline (iCCDM)~\cite{ding2026iccdm}
        &
        64
        &
        12.82
        &
        0.501
        &
        4.891
        \\

        + CVAE (no $\mathcal{L}_{aux}$)
        &
        64
        &
        26.83
        &
        0.270
        &
        5.805
        \\

        + Aux. Regularization
        &
        64
        &
        23.46
        &
        0.257
        &
        5.504
        \\

        + Model-Guidance (MG)
        &
        32
        &
        44.62
        &
        0.272
        &
        3.950
        \\

        + Distillation (16)
        &
        16
        &
        81.84
        &
        0.412
        &
        5.246
        \\

        + Distillation (8)
        &
        8
        &
        139.66
        &
        0.369
        &
        4.038
        \\

        + Distillation (4)
        &
        4
        &
        609.76
        &
        0.358
        &
        5.596
        \\

        + Distillation (2)
        &
        2
        &
        840.36
        &
        0.576
        &
        12.039
        \\

        \hline

    \end{tabular}

\end{table}

RC-49~\cite{ding2023ccgan} is a pose-controlled image generation benchmark containing chair images conditioned on continuous azimuth angles.
It is widely adopted for evaluating continuous viewpoint-conditioned generation.

UTKFace~\cite{utkface} consists of RGB face images with continuous age labels ranging from 1 to 60 years.
The naturally imbalanced age distribution makes it a challenging benchmark for evaluating continuous label consistency and robustness.

The Steering Angle dataset~\cite{steeringangle,steeringangle2}, preprocessed following Ding \textit{et al.}~\cite{ding2023ccgan}, contains real-world driving images conditioned on continuous steering angles.
It provides a practical physical control scenario with complex visual variations and non-uniform continuous label distributions.

\begin{table}[H]
    \centering

    \caption{MG vs. CFG under 32-step sampling.}
    \label{tab:mg_vs_cfg}

    \setlength{\tabcolsep}{3pt}
    \renewcommand{\arraystretch}{1.10}
    \footnotesize

    \resizebox{\columnwidth}{!}{%
    \begin{tabular}{
        p{1.25cm}:p{1.45cm}:cc:ccc
    }

        \hline

        \textbf{Dataset}
        &
        \textbf{Guidance}
        &
        \textbf{NFEs}$\downarrow$
        &
        \textbf{img/s}$\uparrow$
        &
        \textbf{SFID}$\downarrow$
        &
        \textbf{Div.}$\uparrow$
        &
        \textbf{LS}$\downarrow$
        \\

        \hline

        \multirow{2}{*}{Steering}
        & CFG & 64 & 23.46 & 0.257 & 1.072 & 5.504 \\

        & MG (Ours) & 32 & 44.62 & 0.272 & 1.053 & 3.950 \\

        \cdashline{1-7}

        \multirow{2}{*}{RC-49}
        & CFG & 64 & 38.66 & 0.045 & 3.708 & 1.040 \\

        & MG (Ours) & 32 & 59.34 & 0.049 & 3.697 & 0.944 \\

        \cdashline{1-7}

        \multirow{2}{*}{UTKFace}
        & CFG & 64 & 36.60 & 0.360 & 1.154 & 5.219 \\

        & MG (Ours) & 32 & 54.13 & 0.350 & 1.217 & 5.195 \\

        \hline

    \end{tabular}%
    }

\end{table}

{\setlength{\parindent}{0cm}\textbf{Evaluation Metrics.}}
Following CCGM protocols~\cite{ding2023ccgan,ding2026iccdm},
we evaluate performance using SFID, NIQE, Diversity, and
Label Score (LS), which measure conditional fidelity, perceptual quality,
intra-condition variation, and label consistency, respectively.
We report inference throughput in images per second (img/sec) and
aggregate peak GPU memory in GiB, measured on a workstation equipped
with NVIDIA RTX 4090 GPUs. Models exceeding single-GPU memory capacity
are distributed across multiple GPUs, with memory usage summed over
all participating devices.

\begin{table}[H]
    \centering

    \caption{NDA ablation for $128\times128$ latent diffusion.}
    \label{tab:nda_ablation}

    \setlength{\tabcolsep}{5pt}
    \renewcommand{\arraystretch}{1.15}

    \setlength{\arrayrulewidth}{0.25pt}

    \setlength{\dashlinedash}{2.2pt}
    \setlength{\dashlinegap}{2.5pt}

    \tiny

    \resizebox{\columnwidth}{!}{%
    \begin{tabular}{c:l:cc}

        \hline

        \textbf{Dataset}
        &
        \textbf{Configuration}
        &
        \textbf{SFID}$\downarrow$
        &
        \textbf{NIQE}$\downarrow$
        \\

        \hline

        \multirow{2}{*}{UTKFace}
        &
        Baseline (w/o NDA)
        &
        0.286
        &
        1.422
        \\

        &
        NDA
        &
        0.275
        &
        1.180
        \\

        \cdashline{1-4}

        \multirow{2}{*}{Steering Angle}
        &
        Baseline (w/o NDA)
        &
        0.371
        &
        2.817
        \\

        &
        NDA
        &
        0.316
        &
        2.584
        \\

        \hline

    \end{tabular}%
    }

\end{table}

{\setlength{\parindent}{0cm}\textbf{Compared Methods.}}
We include text-to-image diffusion models as baselines:
SD~3 Medium~\cite{esser2024scaling},
FLUX.1~\cite{labs2025flux1kontextflowmatching,flux2024},
Qwen-Image~\cite{wu2025qwen}, and MIRO~\cite{dufour2026miro},
fine-tuned on benchmarks, comparing with GAN-based CCGM methods:
CcGAN (SVDL+ILI)~\cite{ding2023ccgan},
Dual-NDA~\cite{ding2024turning}, and
CcGAN-AVAR~\cite{ding2025imbalance}, and diffusion-based CCGM methods:
CcDPM~\cite{zhao2024ccdpm}, CCDM~\cite{ding2026ccdm} and iCCDM~\cite{ding2026iccdm}.
We extend diffusion-based CCGM with latent manifold purification,
Model-Guidance, and progressive distillation for efficient sampling.

\begin{table}[H]
    \centering

    \caption{Cross-dataset comparison on $64\times64$ benchmarks.
    NFEs indicate the number of function evaluations.}
    \label{tab:distillation_comparison}

    \setlength{\tabcolsep}{3pt}
    \renewcommand{\arraystretch}{1.15}
    \footnotesize

    \resizebox{\columnwidth}{!}{%
    \begin{tabular}{
        p{1.15cm}:p{1.75cm}:cc:ccc
    }

        \hline

        \textbf{Dataset}
        &
        \textbf{Method}
        &
        \textbf{NFEs}$\downarrow$
        &
        \textbf{img/s}$\uparrow$
        &
        \textbf{SFID}$\downarrow$
        &
        \textbf{Div.}$\uparrow$
        &
        \textbf{LS}$\downarrow$
        \\

        \hline

        \multirow{4}{*}{\makecell{Steering\\Angle}}
        & LCM    & 4 & 472.64 & 0.886 & 0.923 & 8.257 \\
        & PCM    & 4 & 368.61 & 0.695 & 0.909 & 4.343 \\
        & Ours-2 & 2 & 840.36 & 0.576 & 1.133 & 12.039 \\
        & Ours-4 & 4 & 609.76 & 0.358 & 1.048 & 5.596 \\

        \cdashline{1-7}

        \multirow{4}{*}{RC-49}
        & LCM    & 4 & 766.35  & 0.151 & 3.498 & 2.074 \\
        & PCM    & 4 & 1107.22 & 0.469 & 2.714 & 1.626 \\
        & Ours-4 & 4 & 1577.19 & 0.114 & 3.565 & 1.848 \\
        & Ours-8 & 8 & 1119.83 & 0.056 & 3.685 & 1.518 \\
        & Ours-16 & 16 & 370.72 & 0.048 & 3.706 & 1.068 \\

        \cdashline{1-7}

        \multirow{4}{*}{UTKFace}
        & LCM     & 4  & 671.89  & 1.478 & 0.602 & 3.482 \\
        & PCM     & 4  & 5172.40 & 0.952 & 0.873 & 4.473 \\
        & Ours-4  & 4  & 1472.39 & 0.770 & 0.787 & 4.046 \\
        & Ours-16 & 16 & 355.48  & 0.344 & 1.206 & 5.608 \\

        \hline

    \end{tabular}%
    }

\end{table}
\vspace{\baselineskip}

{\setlength{\parindent}{0cm}\textbf{Implementation Setup.}}

The CVAE-based latent manifold is first trained with the proposed
purification strategy and then frozen for diffusion training.
The latent spatial compression factor is selected according to the
image resolution, using either $4\times$ or $8\times$ compression.
Unless otherwise specified, the shared diffusion and optimization
settings follow iCCDM~\cite{ding2026iccdm}.

For auxiliary semantic regularization, continuous labels are
min--max normalized to $[0,1]$. We use dataset-independent settings
across all experiments, with $\lambda_{\rm aux}=1.0$ and
$\epsilon_{\rm aux}=0.05$, where $\epsilon_{\rm aux}$ corresponds
to a tolerance of $5\%$ of the normalized label range.

For adaptive vicinal modeling, $N_{\rm AV}$ is selected per dataset
according to its continuous label distribution, following
CcGAN-AVAR~\cite{ding2025imbalance}. NDA uses the worst $20\%$ of
NIQE-ranked reconstructions as negatives and follows the Dual-NDA
penalty configuration~\cite{ding2024turning}. During diffusion training, an EMA teacher with decay $0.9999$
provides Model-Guidance targets.
The coefficient $w_{\rm MG}$ is an empirical training parameter,
distinct from the inference-time CFG scale.
The algebraic correspondence in Eq.~\eqref{eq:mg_clean_prediction}
does not imply matched guidance strength for the reported
MG and CFG configurations.

For efficient sampling, the 32-step teacher is progressively distilled
down to a 4-step student~\cite{salimans2022progressive}. For the
text-to-image baselines, we fine-tune checkpoints of
SD~3 Medium, FLUX.1, Qwen-Image, and MIRO using
LoRA~\cite{hu2022lora} on the corresponding RGB datasets under
the same evaluation protocol.


























\subsection{Experimental Results}
\label{subsec:experimental_results}

We evaluate CAViLD on eight experimental settings covering three continuous conditional generation benchmarks with resolutions ranging from $64\times64$ to $256\times256$.
Quantitative comparisons are summarized in Tables~\ref{tab:low_res_comparison} and~\ref{tab:high_res_comparison}.

\textbf{Overall Performance.}
CAViLD achieves competitive or superior performance across most experimental settings, demonstrating the effectiveness of latent-space modeling, Model-Guidance, and progressive distillation for continuous conditional generation.

Compared with pixel-space diffusion baselines, CAViLD substantially improves sampling efficiency while preserving generation fidelity.
For example, on Steering Angle $64\times64$, our method achieves an SFID of 0.358 with 609.76 img/sec, providing a significant acceleration over iCCDM.
On UTKFace $64\times64$, CAViLD further improves SFID from 0.353 (iCCDM) to 0.344 while achieving much higher throughput.

The framework also scales well to higher resolutions. On UTKFace $256\times256$, CAViLD achieves an SFID of 0.196, matching the strongest competitors while retaining diffusion-based controllability. On Steering Angle $256\times256$, it obtains an SFID of 0.371 and the best Label Score among diffusion-based methods, indicating strong continuous condition adherence. Compared with GAN- and diffusion-based CCGM baselines, CAViLD offers a favorable quality--efficiency balance while substantially reducing iterative sampling cost. Overall, these results demonstrate an effective trade-off among generation fidelity, continuous condition controllability, and practical sampling efficiency.


\subsection{Ablation Study}
\label{subsec:ablation}

In addition to the main comparisons, we conduct ablation studies on
Steering Angle and UTKFace to evaluate latent-space modeling, vicinal
learning, auxiliary regularization, Model-Guidance (MG), progressive
distillation, and Negative Data Augmentation (NDA).
Table~\ref{tab:ablation_sa64} shows that introducing the CVAE latent
representation substantially improves both generation quality and sampling
efficiency over pixel-space iCCDM, while auxiliary regularization further
improves label consistency. Additional CVAE-level results in
Appendix~\ref{app:cvae_ablation} further directly isolate the effect of vicinal
learning on representation quality and label consistency.
MG reduces the required NFEs by internalizing the guidance direction during
training, and progressive distillation compresses the sampler from 32 to
4 steps while maintaining overall competitive quality. Further reducing the sampler
to 2 steps causes clear degradation in SFID and Label Score, indicating the
limit of trajectory compression. Table~\ref{tab:mg_vs_cfg} confirms that MG
achieves comparable or better conditional generation with fewer NFEs than
Classifier-Free Guidance, while Table~\ref{tab:distillation_comparison}
shows a favorable quality--efficiency trade-off over
LCM~\cite{Luo2023LatentCM} and PCM~\cite{Wang2024PhasedCM}.
Finally, Table~\ref{tab:nda_ablation} shows that NDA improves generation
quality at $128\times128$.
These combined results support the joint use of the proposed components in CAViLD,
with the final configurations selected primarily according to SFID and
sampling efficiency.

\section{Conclusion}
\label{sec:conclusion}

In this paper, we propose CAViLD, a latent diffusion framework for efficient continuous conditional image generation.
By transferring the generation process from pixel space to a purified CVAE latent space, CAViLD improves the quality and condition awareness
of the learned latent representation under sparse and imbalanced conditions.
Furthermore, Model-Guidance is incorporated into the training process to eliminate the additional unconditional inference required by conventional classifier-free guidance, while progressive distillation further accelerates sampling.

Extensive experiments on RC-49, UTKFace, and Steering Angle datasets with resolutions up to $256\times256$ demonstrate that CAViLD achieves a favorable balance between generation fidelity, continuous condition consistency, and sampling efficiency.
These results validate the effectiveness of latent-space modeling, guidance-aware training, and trajectory distillation for scalable continuous conditional image generation.

\bibliographystyle{IEEEtran}
\bibliography{references}

\clearpage
\appendices

\section{Vicinal Risk Bound}
\label{app:vicinal_analysis}

This appendix supplies the statistical argument for
Section~\ref{subsec:cvae}. All expectations are conditioned on the observed
labels $\{y_i\}_{i=1}^{N_{\rm data}}$ and parameters $\Theta$
independent of the evaluated samples.
To include the CVAE reparameterization noise, define
\begin{equation}
\ell_{\Theta}(x,y)
=
\mathbb{E}_{\xi}
\left[
\ell_{\Theta}(x,y;\xi)
\right].
\label{eq:posterior_averaged_loss}
\end{equation}
Let $P_y$ be the conditional image distribution. Define the conditional risk
\begin{equation}
\mathcal{R}(y)
=
\mathbb{E}_{X\sim P_y}
\left[
\ell_{\Theta}(X,y)
\right].
\label{eq:conditional_risk}
\end{equation}
For a query label $y^\star$, define its vicinal approximation and empirical
estimator as
\begin{align}
\mathcal{R}^{v}(y^\star)
&=
\sum_i
\bar W_{i,y^\star}
\mathcal{R}(y_i),
\nonumber\\
\widehat{\mathcal{R}}^{v}(y^\star)
&=
\sum_i
\bar W_{i,y^\star}
\ell_{\Theta}(X_i,y_i),
\qquad
X_i\sim P_{y_i}.
\label{eq:appendix_risks}
\end{align}
The effective sample size associated with the normalized vicinal weights is
\begin{equation}
 N_{\mathrm{eff}}(y^\star)
 =
 \left(
 \sum_i\bar W_{i,y^\star}^2
 \right)^{-1}.
 \label{eq:effective_sample_size}
\end{equation}

\noindent\textit{Detailed assumptions and bound.}
Assume that the conditional risk is label-Lipschitz:
\begin{equation}
\left|
\mathcal{R}(y')-\mathcal{R}(y)
\right|
\leq
L|y'-y|.
\label{eq:risk_lipschitz}
\end{equation}
If the $X_i$ are conditionally independent and
\begin{equation}
 \operatorname{Var}_{P_{y_i}}
 \left[
 \ell_{\Theta}(X_i,y_i)
 \right]
 \leq
 \sigma_{\ell}^{2},
\label{eq:loss_variance_bound}
\end{equation}
then

\begin{equation}
\hspace*{-0.5em}
\begin{aligned}
&\mathbb{E}\!\left[
\bigl(\widehat{\mathcal R}^{v}(y^\star)
-\mathcal R(y^\star)\bigr)^2
\right] \\
&\leq
L^2\!\left(
\sum_i \bar W_{i,y^\star}|y_i-y^\star|
\right)^2
+\frac{\sigma_\ell^2}{N_{\mathrm{eff}}(y^\star)} \\
&\leq
L^2\kappa(y^\star)^2
+\frac{\sigma_\ell^2}{|\mathcal N(y^\star)|}
\leq
L^2\kappa(y^\star)^2
+\frac{\sigma_\ell^2}{N_{\mathrm{AV}}}.
\end{aligned}
\label{eq:vicinal_bound}
\end{equation}

\begin{proof}[Proof of Theorem~\ref{thm:vicinal_bound}]
Because $\sum_i\bar W_{i,y^\star}=1$, the label-Lipschitz assumption gives
\begin{align}
\left|
\mathcal{R}^{v}(y^\star)
-\mathcal{R}(y^\star)
\right|
&\leq
\sum_i
\bar W_{i,y^\star}
\left|
\mathcal{R}(y_i)-\mathcal{R}(y^\star)
\right|
\nonumber\\
&\leq
L\sum_i
\bar W_{i,y^\star}|y_i-y^\star|.
\label{eq:vicinal_bias_appendix}
\end{align}
Moreover, conditional independence implies
\begin{align}
 \operatorname{Var}
 \bigl(\widehat{\mathcal{R}}^{v}(y^\star)\bigr)
 &=
 \sum_i\bar W_{i,y^\star}^2
 \operatorname{Var}_{P_{y_i}}
 \bigl[\ell_{\Theta}(X_i,y_i)\bigr]
 \nonumber\\
 &\leq
 \sigma_{\ell}^2\sum_i\bar W_{i,y^\star}^2
 =
 \frac{\sigma_{\ell}^2}{N_{\mathrm{eff}}(y^\star)}.
 \label{eq:vicinal_variance_appendix}
\end{align}
Finally,
$\mathbb{E}[
\widehat{\mathcal{R}}^{v}(y^\star)
-\mathcal{R}^{v}(y^\star)]=0$;
hence the cross term vanishes in the bias--variance decomposition.
Substitution of Eqs.~\eqref{eq:vicinal_bias_appendix} and
\eqref{eq:vicinal_variance_appendix} proves
Eq.~\eqref{eq:vicinal_bound}.
\end{proof}

The first term increases with neighborhood width, whereas the second decreases
as the weights spread over more examples. The bound therefore provides a
bias--variance motivation for the adaptive neighborhood in
Eq.~\eqref{eq:adaptive_radius}: dense regions favor locality, while sparse regions generally require wider label neighborhoods
for fixed $N_{\rm AV}$.

In implementation, drawing
$I\sim\operatorname{Cat}(\{\bar W_{i,y^\star}\})$ gives
\begin{equation}
 \mathbb{E}_{I}
 \bigl[\ell_{\Theta}(X_I,y_I)\bigr]
 =
 \sum_i
 \bar W_{i,y^\star}\ell_{\Theta}(X_i,y_i).
 \label{eq:categorical_unbiased}
\end{equation}
Thus categorical neighbor sampling is an unbiased Monte Carlo estimator of
the fully weighted vicinal objective characterized above, while introducing
the ordinary stochastic variance associated with minibatch optimization.

\section{Ablation Study on Purified CVAE Latent-Space Learning}
\label{app:cvae_ablation}
We compare latent diffusion pipelines using CVAEs trained
under different vicinal-sampling and auxiliary-regularization
settings on Steering Angle. The reported metrics evaluate
the quality and label consistency of the final generated images.

\begin{figure}[!htbp]
    \centering
    \includegraphics[
        width=0.6\columnwidth
    ]{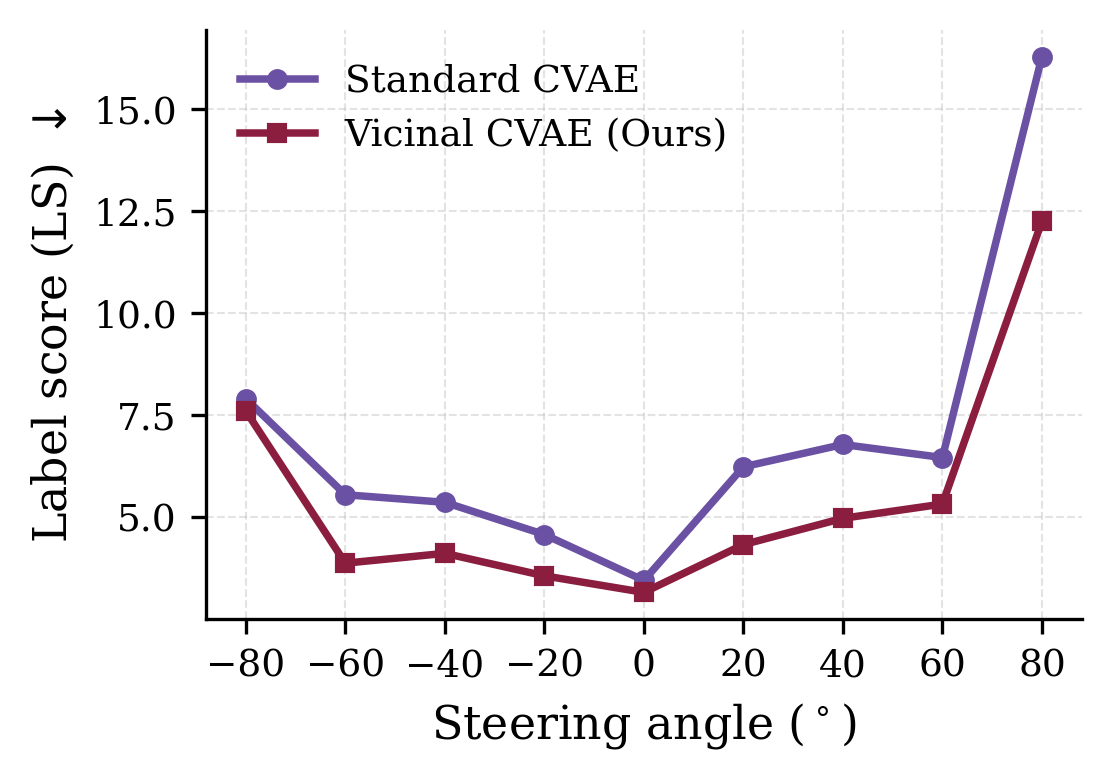}
    \caption{
    Label consistency comparison between the standard CVAE and the proposed
Vicinal CVAE on Steering Angle. Vicinal sampling achieves lower Label Scores
across most conditions, particularly in sparse extreme-angle regions.
    }
    \label{fig:sa64_ls_compare}
\end{figure}

\begin{table}[H]
    \centering

    \caption{Ablation of vicinal learning and auxiliary regularization
    in the $64\times64$ CVAE on Steering Angle.}
    \label{tab:cvae_ablation}

    \setlength{\tabcolsep}{5pt}
    \renewcommand{\arraystretch}{1.15}

    \setlength{\arrayrulewidth}{0.25pt}

    \setlength{\dashlinedash}{2.2pt}
    \setlength{\dashlinegap}{2.5pt}

    \small

    \resizebox{\columnwidth}{!}{%
    \begin{tabular}{c:l:ccc}

        \hline

        \textbf{Dataset}
        &
        \textbf{Configuration}
        &
        \textbf{SFID}$\downarrow$
        &
        \textbf{LS}$\downarrow$
        &
        \textbf{Diversity}$\uparrow$
        \\

        \hline

        \multirow{3}{*}{Steering Angle}
        &
        CVAE w/o vicinal
        &
        0.260
        &
        6.890
        &
        1.065
        \\

        &
        CVAE w/ vicinal
        &
        0.255
        &
        5.749
        &
        1.023
        \\

        &
        CVAE w/ vicinal \& $\mathcal{L}_{\mathrm{aux}}$
        &
        0.257
        &
        5.504
        &
        1.072
        \\

        \hline

    \end{tabular}%
    }

\end{table}

\section{Matrix Score, Probability-Flow ODE, and Model Guidance}
\label{app:matrix_diffusion}

Following the notation in the main text, we omit the superscript
$v$ from the vicinal latent distributions and posterior means. Hence, all
conditional densities and posterior means in this appendix are understood
with respect to the vicinal latent distribution induced by the Stage-II
sampling procedure.

For compactness, write
\[
K = K_{\sigma}(y),
\qquad
z = z_{\sigma}.
\]
For every fixed requested condition $y$, we assume that
$K_{\sigma}(y)$ is continuously differentiable in $\sigma$ and satisfies
\begin{equation}
 K_{\sigma}(y) \succ 0,
 \qquad
 K'_{\sigma}(y) \succeq 0,
 \qquad
 K_0(y)=0,
 \qquad
 \sigma>0.
 \label{eq:matrix_covariance_assumptions}
\end{equation}
All matrix square roots below denote principal symmetric square roots.
For the diagonal covariance schedule used in the main text,
\[
 K_{\sigma}(y)
 =
 \sigma^2 I
 +
 \lambda_{\rm cov}\sigma
 \operatorname{Diag}\!\bigl(a_{\nu}(y)\bigr),
\]
these conditions hold for $\lambda_{\rm cov}\geq0$ and
$a_{\nu}(y)>0$. Because $K_{\sigma}(y)$ is diagonal, all matrix
functions reduce to coordinate-wise operations.

\subsection{Matrix score and probability-flow ODE}

The corrupted density is the Gaussian convolution
\begin{equation}
 p_{\sigma}(z\mid y)
 =
 \int
 p_0(z_0\mid y)
 \mathcal{N}(z;z_0,K)\,dz_0.
 \label{eq:gaussian_convolution}
\end{equation}
Define the exact clean-latent posterior mean associated with this vicinal
conditional distribution as
\begin{equation}
 D^{*}(z,\sigma,y)
 =
 \mathbb{E}
 \left[
 z_0
 \mid
 z,\sigma,y
 \right].
 \label{eq:appendix_exact_posterior_mean}
\end{equation}

Since $K$ is independent of $z$, differentiation under the integral gives
\begin{align}
 \nabla_z p_{\sigma}(z\mid y)
 &=
 \int
 p_0(z_0\mid y)
 \mathcal{N}(z;z_0,K)
 K^{-1}(z_0-z)\,dz_0
 \nonumber\\
 &=
 p_{\sigma}(z\mid y)
 K^{-1}
 \left(
 D^{*}(z,\sigma,y)-z
 \right).
 \label{eq:score_derivation}
\end{align}
Dividing by $p_{\sigma}(z\mid y)$ yields the matrix-form Tweedie
identity
\begin{equation}
 \nabla_z\log p_{\sigma}(z\mid y)
 =
 K_{\sigma}(y)^{-1}
 \left(
 D^{*}(z,\sigma,y)-z
 \right),
 \label{eq:matrix_tweedie_appendix}
\end{equation}
This matrix-form Tweedie identity will be used below to derive
Eq.~\eqref{eq:latent_pfode}.

An additive Gaussian path with covariance $K_{\sigma}(y)$ is generated by
the forward SDE
\begin{equation}
 dz_{\sigma}
 =
 \left(K'_{\sigma}(y)\right)^{1/2}
 dW_{\sigma}.
 \label{eq:forward_matrix_sde}
\end{equation}
Indeed, integrating the instantaneous covariance from $0$ to $\sigma$
gives
\begin{equation}
 \int_0^{\sigma}
 K'_u(y)\,du
 =
 K_{\sigma}(y)-K_0(y)
 =
 K_{\sigma}(y).
 \label{eq:integrated_covariance}
\end{equation}

Because $K'_{\sigma}(y)$ is independent of $z$, the corresponding
Fokker--Planck equation is
\begin{equation}
 \partial_{\sigma}p_{\sigma}
 =
 \tfrac12
 \nabla_z\!\cdot
 \left(
 K'_{\sigma}(y)
 \nabla_z p_{\sigma}
 \right).
 \label{eq:fokker_planck}
\end{equation}
Consider the deterministic velocity field
\begin{equation}
 v_{\sigma}(z,y)
 =
 -\tfrac12
 K'_{\sigma}(y)
 \nabla_z\log p_{\sigma}(z\mid y).
 \label{eq:pfode_velocity}
\end{equation}
It satisfies
\begin{align}
 -\nabla_z\!\cdot
 \left(
 p_{\sigma}v_{\sigma}
 \right)
 &=
 \tfrac12
 \nabla_z\!\cdot
 \left(
 p_{\sigma}
 K'_{\sigma}(y)
 \nabla_z\log p_{\sigma}
 \right)
 \nonumber\\
 &=
 \tfrac12
 \nabla_z\!\cdot
 \left(
 K'_{\sigma}(y)
 \nabla_zp_{\sigma}
 \right),
 \label{eq:continuity_fokker_planck_equivalence}
\end{align}
which agrees with Eq.~\eqref{eq:fokker_planck}.

Substituting Eq.~\eqref{eq:matrix_tweedie_appendix} into
Eq.~\eqref{eq:pfode_velocity} gives
\begin{equation}
 v_{\sigma}(z,y)
 =
 \tfrac12
 K'_{\sigma}(y)
 K_{\sigma}(y)^{-1}
 \left(
 z-D^{*}(z,\sigma,y)
 \right).
 \label{eq:pfode_derivation}
\end{equation}
Therefore, the exact probability-flow ODE associated with the vicinal
conditional latent distribution is
\begin{equation}
 \frac{dz_{\sigma}}{d\sigma}
 =
 B_{\sigma,y}
 \left(
 z_{\sigma}
 -
 D^{*}(z_{\sigma},\sigma,y)
 \right),
 \quad
 B_{\sigma,y}
 =
 \tfrac12
 K'_{\sigma}(y)K_{\sigma}(y)^{-1}.
 \label{eq:appendix_exact_latent_pfode}
\end{equation}
This establishes Eq.~\eqref{eq:latent_pfode} in the main text.

Replacing $D^{*}$ by the learned model-guided denoiser yields the
operational guided sampling field in
Eq.~\eqref{eq:latent_guided_ode}. This replacement does not, in general,
preserve the unguided vicinal conditional marginals exactly.

To clarify the large-noise boundary approximation, define
\begin{equation}
 m_0(y)
 =
 \mathbb{E}[z_0\mid y],
 \qquad
 S_0(y)
 =
 \operatorname{Cov}(z_0\mid y).
 \label{eq:clean_latent_moments}
\end{equation}
The corrupted latent has mean $m_0(y)$ and covariance
$S_0(y)+K_{\sigma}(y)$. Consequently, if the scaled clean latents are
approximately centered and
\begin{equation}
 \left\|
 K_{\sigma_{\max}}(y)^{-1/2}
 S_0(y)
 K_{\sigma_{\max}}(y)^{-1/2}
 \right\|_2
 \ll 1,
 \label{eq:boundary_dominance_condition}
\end{equation}
then the additive Gaussian component dominates the clean-latent
variability, motivating
\begin{equation}
 p_{\sigma_{\max}}(z\mid y)
 \approx
 \mathcal{N}
 \left(
 0,
 K_{\sigma_{\max}}(y)
 \right).
 \label{eq:appendix_gaussian_boundary}
\end{equation}

\subsection{Model-guidance target and normalized branchwise regression}

For model guidance, the conditional and null-conditioned teacher branches
are evaluated using the same noisy state, noise level, and requested
covariance geometry $K_{\sigma}(y)$. We write
\begin{equation}
 \Delta_{\bar\theta}^{\mathrm{ref}}
 =
 D_{\bar\theta}(z_{\sigma},\sigma,y)
 -
 D_{\bar\theta}^{(y)}
 (z_{\sigma},\sigma,\emptyset).
 \label{eq:reference_guidance_appendix}
\end{equation}
The superscript $(y)$ indicates that the explicit condition embedding is
replaced by the null embedding $e_{\emptyset}$, while the covariance and
preconditioning geometry associated with the requested label $y$ remains
fixed.

Because the corruption geometry depends on $y$, the null-conditioned
branch is not identified with the exact denoiser of a globally
label-marginal noisy density. It is instead an operational reference
regression function under the fixed requested geometry. Consequently,
Eq.~\eqref{eq:mg_clean_prediction} defines the operational two-branch
teacher prediction used to construct the target
\begin{equation}
 z_0^{\mathrm{MG}}
 =
 z_0
 +
 w_{\mathrm{MG}}
 \operatorname{sg}
 \left[
 \Delta_{\bar\theta}^{\mathrm{ref}}
 \right].
 \label{eq:appendix_mg_target}
\end{equation}
This is equivalent to Eq.~\eqref{eq:mg_target}.

During each online-model update, the teacher is fixed.
We analyze ideal branch functions with the label representation
and normalization statistics held fixed.
Let $(Z_0,Z_\sigma,Y)$ follow the corresponding Stage-II
sampling distributions. For $\sigma>0$, define
\[
U=C_{\rm in}(\sigma,Y)Z_\sigma,
\qquad T=Z_0^{\rm MG},
\]
\[
H=C_{\rm out}(\sigma,Y)^{-1}
\bigl[T-C_{\rm skip}(\sigma,Y)Z_\sigma\bigr].
\]
Assume $0<p_{\rm drop}<1$, distinguishable conditional
and null embeddings, finite weighted second moments,
and that $e_Y$ determines $Y$ on the training support.

Let $f_c(u,\sigma,e)$ and $f_\emptyset(u,\sigma)$
denote the normalized outputs of $F_\theta$.
Output normalization gives the conditional loss metric
$C_{\rm out}$ and the null loss metric
$(s_z^{-2}+\sigma^{-2})^{1/2}C_{\rm out}^2$.
The conditional metric is determined by the available inputs,
and the scalar null factor depends only on $\sigma$.
Hence, over freely varying branch functions,
the population optima are unique almost surely:
\[
f_c^{\rm opt}(u,\sigma,e)
=
\mathbb E_c[H\mid U=u,\sigma,e_Y=e],
\]
\[
\begin{aligned}
f_\emptyset^{\rm opt}(u,\sigma)
={}&
\left(
\mathbb E_\emptyset
[C_{\rm out}^2\mid U=u,\sigma]
\right)^{-1}\\
&\cdot
\mathbb E_\emptyset
[C_{\rm out}^2H\mid U=u,\sigma].
\end{aligned}
\]
Here $C_{\rm out}=C_{\rm out}(\sigma,Y)$ inside
the expectations.

Returning through Eq.~(8), the null prediction is
\[
\begin{aligned}
D_\emptyset^{\rm opt,(y)}(z,\sigma)
={}&C_{\rm skip}(\sigma,y)z\\
&+C_{\rm out}(\sigma,y)
f_\emptyset^{\rm opt}
\bigl(C_{\rm in}(\sigma,y)z,\sigma\bigr).
\end{aligned}
\]
Under the stated information-sufficiency assumption,
the conditional denoiser optimum reduces to
\[
D_c^{\rm opt}(z,\sigma,y)
=
D^*(z,\sigma,y)
+w_{\rm MG}\Delta_{\bar\theta}^{\rm ref},
\]
where $w_{\rm MG}$ is its active value,
including zero during warm-up.

Attaining both function-space optima requires joint
realizability by the shared network; finite capacity
and optimization may introduce error.
This analysis does not establish convergence of the
coupled online--EMA dynamics.
The null prediction remains an operational reference,
not necessarily the exact global label-marginal denoiser.

\section{Matrix Preconditioning and Loss Weight}
\label{app:preconditioning}

Consider the local Gaussian preconditioning approximation
\[
 z_0\sim\mathcal{N}(0,S_z),
 \qquad
 z=z_0+\varepsilon,
 \qquad
 \varepsilon\sim\mathcal{N}(0,K_{\sigma}),
\]
where $S_z$ and $K_{\sigma}$ are diagonal and therefore commute.
Gaussian conditioning gives
\begin{equation}
 \mathbb{E}[z_0\mid z]
 =
 S_z(S_z+K_{\sigma})^{-1}z,
 \label{eq:linear_mmse}
\end{equation}
which motivates $C_{\rm skip}$ in
Eq.~\eqref{eq:preconditioning_coefficients}. The input standard deviation is
$(S_z+K_{\sigma})^{1/2}$, so
$C_{\rm in}=(S_z+K_{\sigma})^{-1/2}$ normalizes each coordinate.

The covariance of the residual $z_0-C_{\rm skip}z$ is
\begin{align}
 \operatorname{Cov}(z_0-C_{\rm skip}z)
 &=
 S_z-S_z(S_z+K_{\sigma})^{-1}S_z
 \nonumber\\
 &=
 S_zK_{\sigma}(S_z+K_{\sigma})^{-1}
 =
 C_{\rm out}^{2}.
 \label{eq:residual_covariance}
\end{align}
Consequently,
$C_{\rm out}^{-1}(z_0-C_{\rm skip}z)$ has identity covariance under the local
Gaussian approximation. Whitening the clean-prediction residual therefore
uses its inverse covariance,
\begin{equation}
 C_{\rm out}^{-2}
 =
 (S_z+K_{\sigma})(S_zK_{\sigma})^{-1}
 =
 S_z^{-1}+K_{\sigma}^{-1},
 \label{eq:weight_derivation}
\end{equation}
setting $S_z=s_z^2I$ as in the main text gives
$
C_{\rm out}^{-2}
=
s_z^{-2}I+K_\sigma^{-1},
$
which recovers $W_{\sigma,y}$ in
Eq.~\eqref{eq:precision_weight}. Stage-II uses the square root of the corruption precision
as its squared-error metric, whereas distillation uses
the full precision.
For retained conditions, their Euclidean-norm multipliers
are $W_{\sigma,y}^{1/4}$ and $W_{\sigma,y}^{1/2}$,
respectively.
For dropped Stage-II conditions, the multiplier is
$(s_z^{-2}+\sigma^{-2})^{1/4}I$.
\section{Endpoint-Consistent Step Halving}
\label{app:distillation_proof}

We prove endpoint consistency for the two transition rules in our
hybrid distillation schedule:
\[
32
\xrightarrow{\mathrm{Euler}}
16
\xrightarrow{\mathrm{Euler}}
8
\xrightarrow{\mathrm{covariance\ ratio}}
4.
\]
Teacher subscripts denote rollout schedules;
stages inherit online/EMA checkpoint states. The EMA teacher is updated from
the student every 10 optimization steps. All results condition
on the current teacher, held fixed during target construction.

\subsection{Euler-Based Distillation}

For the proof of Theorem~\ref{thm:endpoint_consistency}, let
\[
B_k
=
\tfrac12K'_{\sigma_k}(y)K_{\sigma_k}(y)^{-1},
\qquad
\Delta_k=\sigma_{k+2}-\sigma_k<0,
\]
where $N\in\{32,16\}$ and
$k\in\{0,2,\ldots,N-2\}$.
For $\sigma_k>0$, Eq.~\eqref{eq:latent_covariance} implies that
$K_{\sigma_k}(y)$ and $K'_{\sigma_k}(y)$ are positive diagonal
operators, and hence $B_k$ is nonsingular.

\begin{proof}[Proof of Theorem~\ref{thm:endpoint_consistency}]
For an arbitrary clean-latent prediction $D$, the coarse Euler map is
\begin{equation}
\Psi_k^{\mathrm E}(D)
=
z_{\sigma_k}
+
\Delta_k B_k
\left(z_{\sigma_k}-D\right).
\label{eq:coarse_euler}
\end{equation}
Solving
$\Psi_k^{\mathrm E}(D)=z_{\sigma_{k+2}}^T$
gives
\[
D
=
z_{\sigma_k}
-
B_k^{-1}
\frac{z_{\sigma_{k+2}}^T-z_{\sigma_k}}{\Delta_k}
=
D_k^\dagger.
\]
The solution is unique because $\Delta_k\neq0$ and $B_k$ is
nonsingular. Moreover,
\[
B_k^{-1}
=
2K_{\sigma_k}(y)
\bigl(K'_{\sigma_k}(y)\bigr)^{-1},
\]
which gives Eq.~\eqref{eq:distilled_clean_target}. Direct substitution
then yields
$\Psi_k^{\mathrm E}(D_k^\dagger)=z_{\sigma_{k+2}}^T$.
\end{proof}


Fix $y$ and a positive-noise interval bounded away from zero.
Hold the current EMA teacher fixed, and assume that its
vector field and first derivatives are bounded on the
relevant state region. When all three teacher noise levels
lie in this interval, the two Euler transitions give
\[
\frac{z^T_{\sigma_{k+2}}-z_{\sigma_k}}{\Delta_k}
=
B_k\bigl[
z_{\sigma_k}-D_{\bar\theta_N}(z_{\sigma_k},\sigma_k,y)
\bigr]
+O(|\Delta_k|).
\]
Since $B_k^{-1}$ is bounded on this interval, substitution
into the Euler target gives
\[
D_k^\dagger
=
D_{\bar\theta_N}(z_{\sigma_k},\sigma_k,y)
+O(|\Delta_k|).
\]
Thus, refinement within this interval recovers the current
teacher prediction. This argument excludes zero-noise
terminal steps; their exact algebraic endpoint matching
remains valid.

\subsection{Final Covariance-Ratio Distillation}

For the $8\rightarrow4$ stage, define
\[
R_k
=
\left(
K_{\sigma_{k+2}}(y)
K_{\sigma_k}(y)^{-1}
\right)^{1/2}.
\]
Because $K_\sigma(y)$ is diagonal, all square roots and inverses are
evaluated elementwise.

\begin{proof}[Proof of Theorem~\ref{thm:ratio_endpoint_consistency}]
The corresponding coarse transition is
\begin{equation}
\begin{aligned}
\Psi_k^{\mathrm R}(D)
&=
D+R_k(z_{\sigma_k}-D)
\\
&=
R_kz_{\sigma_k}+(I-R_k)D.
\end{aligned}
\label{eq:coarse_ratio_map}
\end{equation}
Solving
$\Psi_k^{\mathrm R}(D)=z_{\sigma_{k+2}}^T$
gives
\[
D
=
(I-R_k)^{-1}
\left(
z_{\sigma_{k+2}}^T-R_kz_{\sigma_k}
\right)
=
D_{k,\mathrm{ratio}}^\dagger.
\]
The solution is unique by the nonsingularity of $I-R_k$, and direct
substitution proves
\[
\Psi_k^{\mathrm R}
\left(D_{k,\mathrm{ratio}}^\dagger\right)
=
z_{\sigma_{k+2}}^T.
\]
\end{proof}

Since the diagonal entries of $K_\sigma(y)$ decrease when $\sigma$
decreases along the reverse schedule, the entries of $R_k$ lie in
$[0,1)$; hence $I-R_k$ is nonsingular.

Fix $y$ and restrict $\sigma_a,\sigma_b$ to a positive-noise
interval bounded away from zero.
Let $\sigma_b=\sigma_a+\Delta\sigma$.Taylor expansion gives
\begin{equation}
\begin{aligned}
R_{a\rightarrow b}(y)
&=
\bigl(
K_{\sigma_b}(y)K_{\sigma_a}(y)^{-1}
\bigr)^{1/2}
\\
&=
I
+
\tfrac{1}{2}\Delta\sigma\,
K'_{\sigma_a}(y)K_{\sigma_a}(y)^{-1}
+
O(\Delta\sigma^2)
\\
&=
I
+
\Delta\sigma B_{\sigma_a,y}
+
O(\Delta\sigma^2).
\end{aligned}
\label{eq:ratio_local_expansion}
\end{equation}

Consequently,
\begin{equation}
\begin{aligned}
\Psi_{\vartheta}
\bigl(
z_{\sigma_a};
\sigma_a\!\rightarrow\!\sigma_b,y
\bigr)
&\\
&\hspace{-5em}=
z_{\sigma_a}
+
\Delta\sigma B_{\sigma_a,y}
\bigl[
z_{\sigma_a}
-
D_{\vartheta}(z_{\sigma_a},\sigma_a,y)
\bigr]
+
O(\Delta\sigma^2).
\end{aligned}
\label{eq:ratio_transition_consistency}
\end{equation}
giving first-order local consistency with the field defined
by $D_\vartheta$ away from zero noise.
At the terminal step, $R_{a\to0}(y)=0$, so the update
directly returns the denoiser prediction.

Finally, since $W_{\sigma_k,y}\succ0$ in
Eq.~\eqref{eq:distillation_loss}, it reweights the coordinate-wise
optimization errors without altering the pointwise zero-loss targets
$D_k^\dagger$ and $D_{k,\mathrm{ratio}}^\dagger$.
The distilled schedules are nested subsets of the original
32-step schedule:
\[
\mathcal S_{16}
=
\{\sigma_{2j}^{(32)}\}_{j=0}^{16},
\qquad
\mathcal S_{8}
=
\{\sigma_{4j}^{(32)}\}_{j=0}^{8},
\qquad
\mathcal S_{4}
=
\{\sigma_{8j}^{(32)}\}_{j=0}^{4},
\]
where $\sigma_j^{(32)}$ denotes the $j$-th noise level of the
32-step schedule. At each stage, one student transition spans two teacher transitions
and reproduces their endpoint when its prediction matches
the constructed target exactly.
The final four-step sampler applies four covariance-ratio transitions,
each requiring a single conditional denoiser evaluation. It therefore
requires four conditional NFEs in total, followed by one CVAE decoding
operation.



\section{Overall Training and Sampling Procedure}
\label{app:algorithm}

\begingroup
\setlength{\intextsep}{5pt}

\begin{algorithm}[H]
\caption{Purified Vicinal Latent-Space Learning}
\label{alg:latent_training}
\small
\begin{algorithmic}[1]

\REQUIRE
Data $\mathcal{S}=\{(x_i,y_i)\}_{i=1}^{N_{\rm data}}$;
frozen regressor $R_{\omega_R}$; vicinity size $N_{\rm AV}$.

\ENSURE
Frozen encoder $q_\phi$ and decoder $G_\psi$.

\STATE Initialize $Q_{\rm bad}\leftarrow\emptyset$.

\FOR{each CVAE training iteration}

    \STATE Sample $y^\star$; compute its neighborhood and weights
    via Eqs.~\eqref{eq:adaptive_radius}
    and \eqref{eq:vicinal_weight};
    draw $i\sim\operatorname{Cat}(\{\bar W_{i,y^\star}\})$.

    \STATE Encode $z^{(i)}\sim q_\phi(z\mid x_i,y_i)$;
    reconstruct $\widehat{x}_i=G_\psi(z^{(i)},y_i)$.

    \STATE Update $D_\omega$ via
    Eq.~\eqref{eq:nda_discriminator}.

    \STATE Update $(\phi,\psi)$ via
    Eq.~\eqref{eq:cvae_total}.

    \STATE Periodically refresh $Q_{\rm bad}$ with
    the worst NIQE-ranked reconstructions.

\ENDFOR

\STATE Freeze $q_\phi$ and $G_\psi$.

\RETURN $q_\phi$ and $G_\psi$.

\end{algorithmic}
\end{algorithm}

\endgroup
\begingroup
\setlength{\intextsep}{5pt}

\begin{algorithm}[H]
\caption{Matrix-Form Model-Guided Diffusion Training}
\label{alg:mg_training}
\small
\begin{algorithmic}[1]

\REQUIRE
Data $\mathcal{S}$; frozen encoder $q_\phi$ and covariance
network $h_\nu$; noise schedule;
$p_{\rm drop}$, $w_{\rm MG}$, $s_{\rm warm}$, $\beta_{\rm EMA}$.

\ENSURE
Online/EMA checkpoint and EMA denoiser $D_{\theta_N}$.

\STATE Initialize $D_\theta$ and EMA teacher
$D_{\bar\theta}\leftarrow D_\theta$.

\FOR{each diffusion training iteration $s$}

    \STATE Sample $(y^\star,i)$ via
    Section~\ref{subsec:cvae}; set $y=y^\star$.

    \STATE Construct $z_0$ via Eq.~\eqref{eq:clean_latent}.

    \STATE Sample $\sigma$ and a dropout mask; form $z_\sigma$.

    \IF{$s\leq s_{\rm warm}$}
        \STATE Set $z_0^{\rm MG}=z_0$.
    \ELSE
        \STATE Construct $z_0^{\rm MG}$
        via Eq.~\eqref{eq:mg_target}.
    \ENDIF

    \STATE Update $\theta$ via Eq.~\eqref{eq:model_guided_loss} using the same mask.

    \STATE EMA update:
    $\bar\theta\leftarrow
    \beta_{\rm EMA}\bar\theta+(1-\beta_{\rm EMA})\theta$.

\ENDFOR

\STATE Set $D_{\theta_N}\leftarrow D_{\bar\theta}$.
\RETURN $D_{\theta_N}$ and $(\theta,\bar\theta)$.

\end{algorithmic}
\end{algorithm}

\endgroup

\begingroup
\setlength{\intextsep}{5pt}

\begin{algorithm}[H]
\caption{Hybrid Step-Halving Distillation}
\label{alg:progressive_distillation}
\small
\begin{algorithmic}[1]

\REQUIRE
Data $\mathcal{S}$; frozen encoder $q_\phi$ and covariance
network $h_\nu$; Stage-II online/EMA checkpoint;
$N=32$ and its noise schedule.

\ENSURE
Four-step EMA denoiser $D_{\theta_4}$.

\STATE Start from the $N$-step checkpoint.

\FOR{$M=N,N/2,\ldots,16$}

    \STATE Restore the student and EMA teacher from the
preceding checkpoint's online and EMA states.

    \FOR{each distillation iteration}

        \STATE Sample $(z_0,y)$ and $k\in\{0,2,\ldots,M-2\}$;
        construct $z_{\sigma_k}$
        via Eq.~\eqref{eq:forward_perturbation}.

        \STATE Compute the teacher endpoint
        via Eq.~\eqref{eq:teacher_endpoint}.

        \STATE Compute $D_k^\dagger$
        via Eq.~\eqref{eq:distilled_clean_target}.

        \STATE Update $D_{\theta_{M/2}}$
        via Eq.~\eqref{eq:distillation_loss}; update EMA.

    \ENDFOR

    \STATE Carry both online and EMA states to the next stage.

\ENDFOR

\STATE Load eight-step online/EMA states; teacher rollouts:
Eq.~\eqref{eq:covariance_ratio_transition}.

\FOR{each final distillation iteration}

    \STATE Sample $(z_0,y)$ and $k\in\{0,2,4,6\}$;
    construct $z_{\sigma_k}$
    via Eq.~\eqref{eq:forward_perturbation}.

    \STATE Compute the teacher endpoint
    via Eq.~\eqref{eq:ratio_teacher_endpoint}.

    \STATE Compute $D_{k,\mathrm{ratio}}^\dagger$
    via Eq.~\eqref{eq:ratio_distillation_target}.

    \STATE Update $D_{\theta_4}$ via
    Eq.~\eqref{eq:distillation_loss}
    with target $D_{k,\mathrm{ratio}}^\dagger$; update EMA.

\ENDFOR

\RETURN the final EMA denoiser as $D_{\theta_4}$.

\end{algorithmic}
\end{algorithm}

\endgroup

\end{document}